\documentclass[jair,twoside,11pt,theapa]{article}
\usepackage{jair, theapa, rawfonts}

\jairheading{1}{1993}{1-15}{6/91}{9/91}
\ShortHeadings{Anytime Stochastic Task and Motion Policies}
{Shah and Srivastava}

\usepackage{moreverb,url}
\usepackage{amsmath}

\usepackage{graphicx}
\usepackage{stmaryrd}

\usepackage{url}
\usepackage{cleveref}
\usepackage[inline,shortlabels]{enumitem}
\usepackage[ruled,vlined,linesnumbered]{algorithm2e}
\usepackage{makecell}
\usepackage{tikz}
\usepackage{amssymb}
\usepackage{siunitx}

\newcommand{\U}{\mathcal{U}\xspace}
\newcommand\M{{\cal M}}        % caligraphic Mß

\newcommand{\set}[1]{\{#1 \}}
\newcommand{\?}{\raisebox{.5pt}{\textcircled{\raisebox{-.9pt} {{\small
          ?}}}}}
\newcommand{\abs}[1]{\left[ #1 \right] }
\newcommand{\Tau}{\mathcal{T}}
\newcommand{\Psym}{\mathcal{P}_\text{sym}}
\newcommand{\Ph}{\mathcal{P}_\text{h}}
\newcommand{\psym}{p_\text{sym}}
\newcommand{\ph}{p_\text{h}}
\newcommand{\Asym}{\mathcal{A}_\text{sym}}
\newcommand{\Ah}{\mathcal{A}_\text{h}}
\newcommand{\Am}{\mathcal{A}}

\newcommand{\Pm}{\mathcal{P}}
\newcommand{\Sm}{\mathcal{S}}
\newcommand{\X}{\mathcal{X}}
\newcommand{\C}{\mathcal{C}}
\newcommand{\citet}[1]{\citeauthor{#1}~\citeyear{#1}}
\newcommand{\inter}[1]{\llbracket #1 \rrbracket}

\usepackage{lipsum}

\newcommand\blfootnote[1]{%
  \begingroup
  \renewcommand\thefootnote{}\footnote{#1}%
  \addtocounter{footnote}{-1}%
  \endgroup
}

\newtheorem{definition}{Definition}

\newtheorem{theorem}{Theorem}
\newtheorem{proof}{Proof}
\usepackage{xcolor}
\usepackage{etoolbox}
\usepackage{caption}
\usepackage{subcaption}
\providetoggle{full}

\newcommand\BibTeX{{\rmfamily B\kern-.05em \textsc{i\kern-.025em b}\kern-.08em
T\kern-.1667em\lower.7ex\hbox{E}\kern-.125emX}}

\def\volumeyear{2021}

\begin{document}

\title{An Anytime Hierarchical Approach for \\ Stochastic Task and Motion Planning}

% \author{Naman Shah\affilnum{1} and Siddharth Srivastava\affilnum{1}}
% \author{Naman Shah\affilnum{1} and Siddharth Srivastava\affilnum{1}}
% \affiliation{\affilnum{1}Arizona State University, Tempe, AZ}

% \corrauth{Naman Shah, 
% School of Computing and Augmented Intelligence (SCAI),
% Arizona State University,
% Tempe, AZ, 85281, USA.}
\author{\name Naman Shah \email namanshah@asu.edu \\
    \addr Arizona State University, \\ 699 S Mill Ave, 
    Tempe, AZ, USA, 85281 
    \AND 
    \name Siddharth Srivastava \email siddharths@asu.edu \\ 
    \addr Arizona State University, \\ 699 S Mill Ave,
    Tempe, AZ, USA, 85281 
}

\maketitle

\begin{abstract}
    In order to solve complex, long-horizon tasks, intelligent robots
    need to carry out high level, abstract planning and reasoning in
    conjunction with motion planning. However, abstract models are
    typically lossy and plans or policies computed using them can be
    inexecutable. These problems are exacerbated in 
    stochastic situations where the robot needs to reason about, and
    plan for multiple  contingencies. 
      
    We present a new approach for integrated task and
    motion planning in stochastic settings. In contrast to prior work in this
    direction, we show that our approach can effectively compute
    integrated task and motion policies whose branching structures
    encode agent behaviors that handle multiple execution-time contingencies. We
    prove that our algorithm is probabilistically complete and can
    compute feasible solution policies in an anytime fashion so that the
    probability of encountering an unresolved contingency decreases over
    time. Empirical results on a set of challenging  problems show the
    utility and scope of our method.
\end{abstract}

% \keywords{ Stochastic task and motion planning, State and entity abstraction, Mobile manipulation, Hierarchical planning }

\blfootnote{In submission}

% \footnote{In submission}

\section{Introduction}

A long-standing goal in robotics is to develop robots that can operate autonomously in real-world environments and solve complex tasks such as cleaning a room or organizing a table. Recent developments in sampling-based motion planning algorithms~\cite{kavraki1996probabilistic,Lavalle98rrt,sucan2012open} have enabled robots to efficiently plan in configuration spaces that have infinite states and a large branching factor. However, these sampling-based motion planners are not designed for to plan over long horizons and changing configuration spaces for solving complex real-world problems. The problem becomes even more challenging when the robot's actions and/or its environment are stochastic, as the agent has to not only deal with a long horizon but also needs to compute a contingent solution that deals with all possible situations that might arise while acting in the real world. E.g., consider a household robot that is arranging a dining table. The robot may to pick up objects or may drop objects while carrying them from one location to another. What should robot do if this happens?

\begin{figure}[t!]
  \centering
  \includegraphics[height=1.5in]{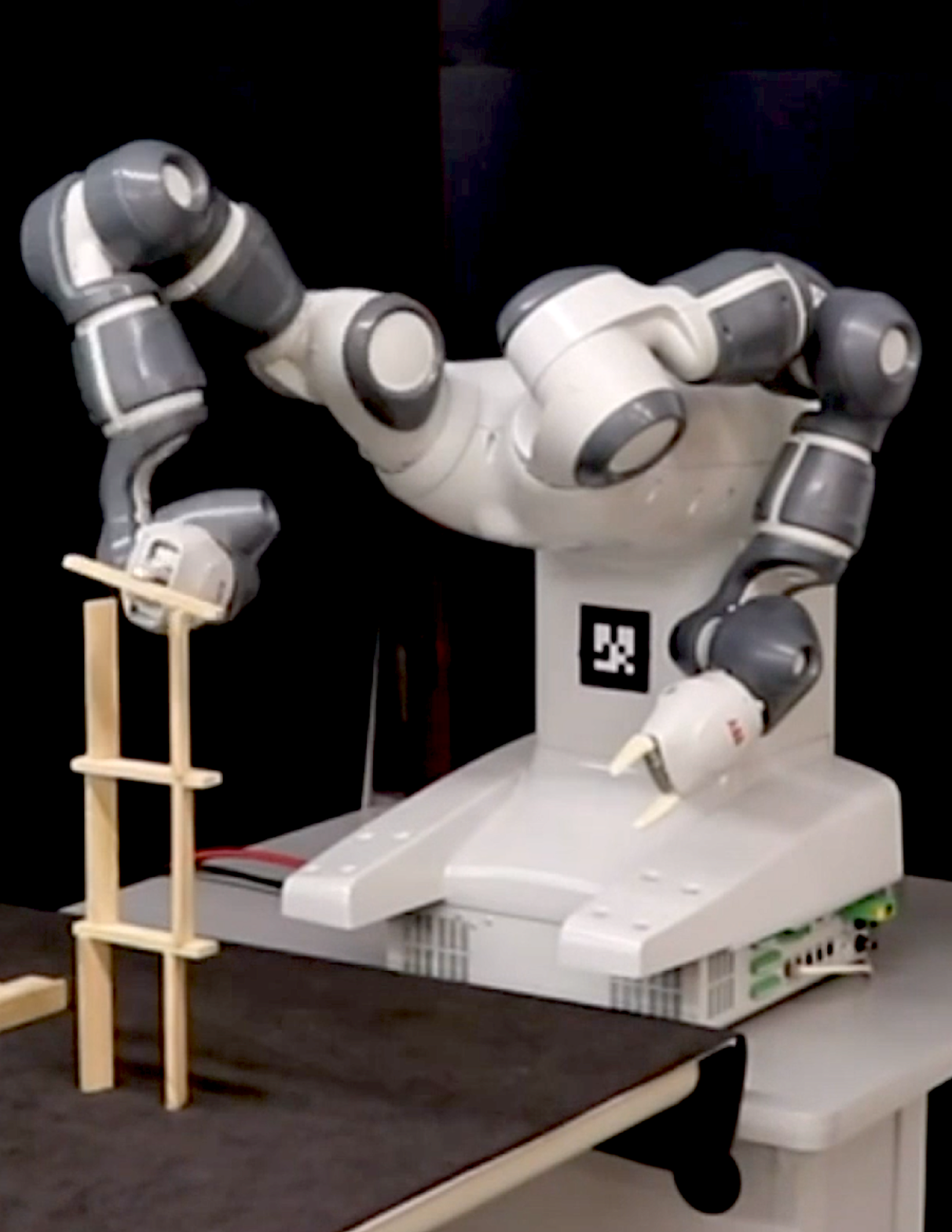}\hspace{1cm}
  \includegraphics[height=1.5in]{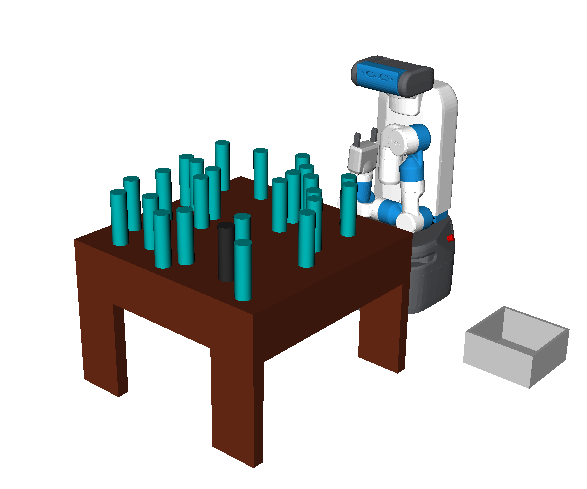}
  \caption{Left: YuMi robot uses the algorithm developed in this paper to build a $3\pi$ structure using Keva planks despite stochasticity in their initial locations. Right:
    A stochastic variant of the cluttered table domain where robot
    is instructed to pick up the black can, but pickups may fail and crush the cans requiring them to be disposed.}
  \label{fig:domainFig}
\end{figure}

A na\"ive approach for overcoming such a problem would be to first compute a symbolic high-level policy using an abstract model of the domain defined using a symbolic language such as Probabilistic Planning Domain Definition Language (PPDDL)~\cite{younes2004ppddl1} or Relational Dynamic Influence Diagram (RDDL)~\cite{sanner10_rddl} and then refining each possible scenario in the policy by computing low-level motion plans for every action in it. This approach is na\"ive in the sense that 
% it neither prioritizes more likely outcomes ahead of outcomes that are highly unlikely (but still possible) nor it guarantees completeness as every action might not admit a low-level motion plan as 
\begin{enumerate*}[a)] \item abstracted models are lossy and may lose important geometric information about the problem. Policies resulting from such approaches might not have any feasible motion planning refinements for some of their actions ~\cite{cambon09_asymov,kaelbling11_hierarchical,srivastava14_tmp} and 
  \item the number of actions in a policy grows exponentially with the horizon. Computing motion planning refinements for the entire policy may be expensive, and it may require large amount of time. 
\end{enumerate*}

Therefore, most solutions approaches to this problem focus on most likely scenarios. They compute task and motion plan for these most likely actions and outcomes and execute it until the robot achieves the goal or reaches a state for which it has not yet planned an action. In the latter situation, they replan from that state and compute a new plan that reaches the goal from it. While this approach may achieve the goal, it may not be safe or sufficient as on-the-fly replanning is prone to errors and may result in unwanted situations or dead ends. E.g., a vacuum cleaner robot may end up in a water puddle and damage itself if such a determinization~\cite{determinization} based approach is employed. 

To the best of our knowledge, this paper presents the first \emph{probabilistically complete} and \emph{anytime} approach for computing integrated task and motion policies for stochastic environments where each action has a discrete set of possible outcomes using off-the-shelf symbolic planners and motion planners.
% Our approach interleaves the computation of refinements with updating of abstractions to compute truly feasible task and motion policies for stochastic environments where each action in the environment has a discrete set of possible outcomes. 
Our approach computes the likelihood of each possible outcome of the policy and weighs it against the estimated cost of computing solution for that outcome and use it prioritize outcomes for refining.
The approach is anytime in the sense that it continually improves the quality of the solution while ensuring that more likely situations are resolved earlier by the approach. It also provides a running estimate of the probability mass of likely executions covered in the current policy. This estimate can be used to start execution based on the level of risk acceptable in a given application, allowing one to trade-off pre-computation time for likelihood of outcomes covered.

This way our approach generalizes the methods of computing solutions for most likely outcomes during execution~\cite{hadfield15_modular,determinization,teichteil2008rff} to the problem of integrated task and motion planning by using the anytime approaches in \emph{AI} planning~\cite{dean88_anytime,zilberstein93_anytime,dean95_anytime}. 
% The presented approach is the first \emph{probabilistically complete} approach that uses sound abstractions to solve combined task and motion planning problems in stochastic environments.
%  It works with arbitrary off-the-shelf symbolic solvers and off-the-shelf motion planners which allows it to scale automatically with improvements in either of these active areas of planning research and can also be used in deterministic settings as a special case.

In contrast to the closest related work (a conference paper by the team)~\cite{shah2020anytime}, this paper includes a new, more general, and rigorous problem formulation with an algorithmic paradigm that provides guarantees of probabilistic completeness, an extensive empirical evaluation on a broader class of test domains including experiments with a new generalized SSP solver, and a thorough analysis of the results.

The rest of the paper is structured as follows: Sec.~\ref{sec:background} provides the formal background; Sec.~\ref{sec:related} discusses recent work on the topic; Sec.~\ref{sec:formal} provides our formal framework and defines the stochastic task and motion planning problem; Sec.~\ref{sec:algo} discusses our overall algorithm and Sec. \ref{sec:empirical} provides the empirical evaluation for our approach.

\section{Background} 
% and Related Work}
\label{sec:background}
We start with discussion on motion planning (Sec.~\ref{subsec:mp}). We then discuss fundamentals of first-order logical models (Sec.~\ref{subsec:first_order}) and abstractions (Sec.~\ref{subsec:abstraction}). Sec.~\ref{subsec:ssp} discusses stochastic shortest path problems in the context of our approach. 
% s the definition of classical planning and discuss a few of the many approaches that perform classical planning. Sections \ref{sec:ssp} and \ref{sec:mp} discuss the stochastic variant of the planning problem and motion planning briefly. Lastly, section \ref{sec:itmp} discusses the recent work done related to the work presented in the paper.

% Combined task and motion planning has achieved a lot of research focus in recent years, though most of the works are focused on solving combined task and motion planning problem in deterministic settings. 

% \subsection{Planning under Uncertainty}

\subsection{Motion Planning}
\label{subsec:mp}

Let $\X = \X_\text{free} \cup \X_\text{obs} $ be the configuration space~\cite{Lav06} of a given robot. Here $\X_\text{free}$ represents the set of configurations where the robot is not in collision with any obstacle and $\X_\text{obs}$ represents configurations in collision with an obstacle. Let $x_i \in \X_\text{free}$ and $x_g \in \X_\text{free}$ be the initial and goal configurations of the robot. A motion planning problem is defined as follows:
\begin{definition}

    A \textbf{motion planning problem} is a $4$-tuple $\langle \X, f, x_o, x_g \rangle$ where, 
    \begin{itemize}
        \item $\X$ is the space of all possible configurations (a.k.a \emph{configuration space or C-space}). \item $x_i$ is the initial configuration.
        \item $x_g$ is the goal configuration. 
        \item $f: \X \rightarrow \{0,1\}$ determines whether a pose $x \in \X$ is in a collision or not. $f(x) = 0$ if the configuration $x$ is in collision ($x \in \X_\text{free}$).  
    \end{itemize}
\end{definition}

A solution to a motion planning problem is a collision-free trajectory $t: [0,1] \rightarrow \X$ such that $t(0) = x_i$ and $t(1) = x_g$. A trajectory is collision-free if $f(x) = 1$ for every configuration in the trajectory.

\subsection{First-Order Logical Models}
\label{subsec:first_order}
We use first-order logic to express models for planning problems.  Let $\tau_R$ be a type of variables that represent vectors of continuous real values and $\tau_O$ be a type of variables that represent names of the objects in the environment and symbolic references for the variables of the type $\tau_R$. Let $\U$ be the universe consisting of object names of type $\tau_O$, continuous vectors of type $\tau_R$, and symbolic references for these vectors, let $\Pm$ be a set of predicates, and let $\C$ be a set of constants. Let $\mathcal{V} = \Pm \cup \C$ define the vocabulary. In this work, we consider two kinds of predicates: symbolic and hybrid.  We define each of them as follows: 
\begin{definition}
    A predicate $\psym(y_1,\dots,y_k) \in \Psym$ is a \textbf{symbolic predicate} iff all of its arguments $y_1,\dots,y_k$ are of type $\tau_O$.
\end{definition} 

\begin{definition}
    A predicate $\ph(y_1,\dots,y_k,\theta_1,\dots,\theta_m) \in \Ph$ is a \textbf{hybrid predicate} iff its arguments $y_1,\dots,y_k$ are of type $\tau_O$ and  $\theta_1,\dots,\theta_m$ are of type $\tau_R$.
\end{definition}

States are logical structures or models defined over predicates. A structure or a state $s \in \Sm$, of vocabulary $\mathcal{V}$ where $\Pm = \Psym \cup \Ph \subset \mathcal{V}$, consists of a universe $\U$, a predicate $p^S$ over $\U$ for every predicate $p \in \mathcal{V}$, and an element $c^S$ over $\mathcal{U}$ for every constant symbol $c \in \mathcal{V}$. An \emph{interpretation} of a predicate $p \in \Pm$ provides a relation between objects in the universe $\U$. 
% E.g., an interpretation, for a hybrid predicate \emph{at} -{}- \emph{at(o$_1$,loc)} -{}- if true, specifies a relation between \emph{o$_1$} and \emph{loc} representing that the object \emph{o$_1$} is at the location \emph{loc}, where $o_1 \in \U$ is the name of an object in the environment and \emph{loc} $\in \mathbb{R}^{6}$ represents the pose of the object in the environment. Similarly, an interpretation, for a symbolic predicate \emph{on} -{}- \emph{on($o_x$,$o_y$)} -{}- if true, specifies a relation between objects $o_x$ and $o_y$ representing that the object $o_x$ is placed on the object $o_y$, where $o_x, o_y \in \mathcal{U}$ are the names of objects in the environment. 
 Henceforth, we use $ \inter{p}_s$ and $\inter{\psi}_{s}$ to denote interpretations of the predicate $p$ and a formula $\psi$ in $s \in \Sm$ respectively. 

We use the notion of actions in PPDDL~\cite{McDermott1998PDDL} to represent the actions available to the robot. We classify actions available to the robot as \emph{symbolic} and \emph{hybrid} actions depending on the types of predicates that appear in the actions' descriptions and their arguments. Symbolic actions only use predicates from $\Psym$ to specify their preconditions and effects, but hybrid actions may use predicates from $\Psym$ and $\Ph$. 

% E.g., a symbolic action \emph{TurnOn(light)} is executed by the robot using a wireless transmitter to turn on a light where the argument \emph{light} is the name of an object in the environment and \emph{IsTurnedOn(light)} is a predicate from $\Psym$ that appears in the effect of the action. A concrete hybrid action \emph{Place(obj, loc, traj)} is executed by the robot by placing an object \emph{obj} at a certain location \emph{loc} using the trajectory \emph{traj}, where \emph{obj} is a name of the object in the environment while \emph{pose} and \emph{traj} are vectors or real values representing target pose of the object and trajectory to be used by the robot to execute the action. \emph{$\lnot$holding(obj)} and \emph{at(obj, pose)} are symbolic and concrete predicates respectively which appear in the effect of the action. 

A hybrid action is a motion planning action if the robot requires to compute a motion plan while executing the hybrid action. Action arguments for motion planning actions specify trajectories required to execute these actions and preconditions can be used to specify constraints on these motion planning trajectories. Values for these motion planning arguments can be sampled using a motion planner. 
% E.g, the action \emph{Place(obj, loc, traj)} contains a motion planning argument \emph{traj} which refers to a valid collision-free trajectory. 
We formally define motion planning actions as follows:
\begin{definition}
A \textbf{motion planning action} $a_\text{mp}(o_1,\dots,o_k,\theta_1,\dots,\theta_j,t_1,\dots,t_n)$ is a hybrid action where $o_1,\dots,o_k$ are of type $\tau_o$, $\theta_1,\dots,\theta_j$ are of type $\tau_R$, and $t_1,\dots,t_n$ are motion planning trajectories. \emph{pre(a$_\text{mp}$)} contains constraints on $t_1,\dots,t_n$ and \emph{eff(a$_\text{mp})$} represents the effective pose of the robot after executing action $a_\text{mp}$. $\mathcal{A}_\text{mp} \subset \Ah$ is the set of all motion planning actions.
\end{definition}  

We now use these concepts of first-order logical models to discuss stochastic shortest path problems in the context of our problem.

\subsection{Abstraction}
\label{subsec:abstraction}

We use the concepts of abstraction to model the robot manipulation problem as a symbolic planning problem. 
Let $V_l$ be a low-level vocabulary and $V_h$ be a high-level vocabulary such that $V_h \subset V_l$; the predicates in $V_h$ are defined as identical to their counterparts in $V_l$. We define \emph{relational abstractions} as first-order queries that map structures over one vocabulary to structures over another vocabulary. A first-order query $\alpha$ from $V_h$ to $V_l$ defines functions in $V_h$(also identified as $\alpha(V_l))$ using the $V_l$-formulas in $S_l$: $\llbracket r \rrbracket _{V_h} (o_1, o_2, ..., o_n) = True$ iff $\llbracket \psi_{r}^{\alpha}(o_1, o_2, ..., o_n) \rrbracket _{S_l} = True$, where $\psi_{r}^{\alpha}$ is a formula over $V_l$. Such abstractions reduce the number of properties being modeled keeping number of objects the same. 
\paragraph{\textbf{Example}} Let $V_l$ be a first-order vocabulary consisting of continuous locations on a tabletop, an object $o$ and a relation \emph{atLocation($o$,$l$)} that defines relationship between an object $o$ and a continuous location $l$ on the tabletop. Let $V_h$ also be a first-order vocabulary that has a $0$-ary relation (constant) \emph{OnTable}. A first-order query $\alpha$ defines the value of \emph{OnTable} in $V_h$ as follows: \emph{OnTable} is true iff there exists a continuous location $l$ such that relation \emph{atLocation($l$,$o$)} is true in $V_l$.

The goal of our approach is to compute a solution for the obtained high-level symbolic problem with ``refinements'' that select a specific motion planning problem and its solution in the concrete space for each action in the high-level solution. In this view, each high-level action corresponds to infinite low-level problems in the concrete space, each defined by a specific initial and target configuration of the robot. For example, a high-level action of placing a cup on a table corresponds to infinite motion planning problems, each defined by a different location of the cup on the table. The refinement process would require selecting one of these problems and computing a valid motion planning solution for it. 

% Before defining the abstraction used in this work, we define some required terminology.
% An interpretation of a predicate $p \in \mathcal{P}$ represents either a property of an object or a relation between two or more objects. E.g., \emph{at($o_1$, loc)} is a hybrid predicate that specifies that object $o_1$ is at the location \emph{loc} where, $o_1\ \in \mathcal{U}$ is a name for an object in the environment and \emph{loc} $\subset \mathbb{R}^6$ represents the pose of the object with name $o_1$ in the $3$-D environment. Similarly, \emph{on($o_x$, $o_y$)} specifies that the object $o_x$ is placed on  the object $o_y$ where, $o_x,o_y \in \mathcal{U}$ are names of objects in the environment.

\subsection{Stochastic Shortest Path Problems}
\label{subsec:ssp}
% \subsubsection{Probabilistic planning}~\\ 
We use first order logical models defined in Sec.~\ref{subsec:first_order} to define stochastic shortest path problems (SSPs)~\cite{bertsekas91_ssp} as follows: 
% We now discuss stochastic version of the planning problem named \emph{stochastic shortest path problem}. \citet{bertsekas91_ssp} defines the \emph{stochastic shortest path problems} as follows: 
\begin{definition}
    A continuous \textbf{stochastic shortest path} (SSP) problem is defined as a $7$-tuple $ P_\text{ssp} = \langle \mathcal{P}, \mathcal{S}, \mathcal{A}, T, C, \gamma, H \rangle$ where,
    \begin{itemize}
        \item $\Pm = \Psym \cup \Ph$ is a set of predicates.
        \item $\Sm$ is a set of states such that every state is a structure defined over $\Pm$.
        \item $\Am = \Asym \cup \Ah$ is a set of actions. Here $\Am_\text{mp} \subset \Ah$.
        \item $T: \Sm \times \Am \times \Sm \rightarrow [0,1]$ is a transition function that assigns a probability value to each transition $(s,a,s')$. Here $s'$ represents a resultant state reached by the robot after executing an action $a$ in state $s$. 
        \item $C: \mathcal{S} \times \mathcal{A} \rightarrow \mathbb{R}$ is a cost function.
        \item $\gamma = 1$ is a discount factor (fixed). 
        \item $H$ is a finite horizon.
    \end{itemize}
\end{definition}
 A solution to an SSP is a non-stationary policy $\pi$ of the form
$\pi:\Sm \times \{1,\dots, H\} \rightarrow \Am$ that maps all
the states and time steps at which they are encountered to an
action. The optimal policy $\pi^{*} $ is a policy that reaches the
goal state with the least expected cumulative cost. Due to the finite horizon, SSP
policies need not be stationary.

\section{Related Work}
\label{sec:related}
% As integrated task and motion planning rests on the intersection of task planning and motion planning, developments in either area affect it. So, now we discuss some related work in each of these areas developed in recent times. 

\paragraph{\textbf{Stochastic Task Planning}} 
Many approaches have been developed for classical planning efficiently in recent years~\cite{blum1997fast,bonet2001planning,hoffmann2001ff}.
%  \cite{bonet2001planning} introduce a way to synthesize \emph{domain-independent} heuristics by relaxing the problem through ignoring predicates in the \emph{delete} list of each action.  The relaxed problem is easier to solve and the solution can be used to estimate heuristics for the states encountered in the solution. GraphPlan~\cite{blum1997fast} generates planning graph for a given planning problem which can be used to automatically synthesize heuristics such as $h_{add}$ and $h_{max}$. \cite{hoffmann2001ff} uses planning graphs to generate tighter heuristic than $h_{add}$ and $h_{max}$  known as $h_{\text{ff}}$ by avoiding re-counting of actions that achieve similar predicates. 
% Most classical planners use a relational language known as \emph{Planning Domain Definition Language (PDDL)}~\cite{McDermott1998PDDL}.
Similarly, numerous approaches have been developed to solve stochastic shortest path problems. Dynamic
programming algorithms such as value iteration and policy iteration can be
used to compute policies for SSPs. 
% But these approaches require computing optimal cost action for each state in the state space to converge which can take a huge amount of time for large state spaces.
 \emph{Real-time dynamic programming} (RTDP)~\cite{barto1993rtdp} generalizes \emph{Korf's Learning-Real-Time-A*} algorithm to a trial-based dynamic programming method that ignores a large part of the state-space by only expanding states encountered in trials to solve $SSPs$ faster. \emph{LAO*}~\cite{hansen2001lao} uses heuristics to expand the partial policy tree along with local value iteration to compute policies for SSPs. \emph{Labeled RTDP}~\cite{bonet2003labeled} extends \emph{RTDP} by labeling states that have converged greedy policy to reduce the number of states considered for expansion to decrease policy computation time. \cite{Muise2012ImprovedNP} use state relevance to guide the search to reduce the time to compute the policy. \cite{larach2019sspdead} provide a method that decomposes an $SSP$ into multiple smaller $SSPs$ and combines the solution to handle \emph{dead ends}.

\paragraph{\textbf{Hierarchical Planning}} Hierarchical approaches~\cite{sacerdoti1974planning,knoblock1990learning,erol1995semantics,seipp2018counterexample} use abstractions to generate different hierarchies of relaxed planning problems in order to compute a solution for a complex planning problem. State abstraction generates hierarchies by removing certain predicates (in relational domains) or variables (in factored domains) from the domain vocabulary. ABSTRIPS \cite{sacerdoti1974planning} is one of the earliest hierarchical planning approaches which assigns a rank to each literal using a predefined order and complexity of achieving that literal in the STRIPS planning process. Abstraction hierarchy is generated by dropping literals from the precondition of actions in the domain in the order specified by the rank of literals. The planning hierarchy generated using ABSTRIPS is common for all problems in the given domain and not catered to independent problems. 

ALPINE \cite{knoblock1990learning} uses \emph{ordered monotonicity} to overcome this issue by generating abstraction hierarchies tailored to each problem for the given domain.  
% While this approach makes strong assumptions of ordered monotonicity and downward refinement property~\cite{bacchus1991downward}, our approach does not require such critical assumptions.
\cite{seipp2013counterexample,seipp2018counterexample} use counter-example guided abstraction refinement (CEGAR) to solve a complex planning problem hierarchically using Cartesian abstraction -{}- a variant of predicate abstraction. This CEGAR-based approach starts with a na\"ive abstraction for the problem and computes an optimal plan for the abstract model. It tries to execute this plan in the original model. If it fails to execute the plan successfully, it computes a flaw in the current plan and uses it to refine the current abstract model. This approach requires a pre-image of each \emph{grounded operator} and a bounded branching factor for the search tree. Such approaches are not conducive to task and motion planning setups because they require discrete action and state spaces while task and motion planning operates in continuous states and action spaces.

Temporal abstractions generate high-level actions that are compositions of multiple low-level actions. Some hierarchical planning approaches employ temporal abstraction to create relaxed problems. Multiple approaches~\cite{kambhampati1998hybrid,bacchus2000using,bercher2014hybrid} have used hierarchical task networks (HTNs)~\cite{erol1995semantics} to compute plans efficiently for complex tasks. HTNs use temporal abstractions to define tasks over primitive actions. The goal is to compute a final plan which is a composition of the high-level tasks that are achieved through the partial order planning of the primitive actions. \citet{marthi2007hla} compute hierarchical domain descriptions based on angelic semantics using temporal abstractions. They use a top-down forward search algorithm to refine the high-level actions into a sequence of primitive actions. While this approach and HTN-based approaches efficiently perform top-down planning using temporal abstraction, they fail to compute accurate plans in the models that do not fulfill downward refinement property. Additionally, they do not handle stochasticity. 
% On the other hand, our approach handles such domains using an interleaved approach that refines actions while continually improving the abstract model.

Several approaches utilize abstraction for solving MDPs (\cite{hostetler14_state,bai16_markovian,li06_abstractMDP,singh95_abstractRL}). However,
these approaches assume that the full, unabstracted MDP can be
efficiently expressed as a discrete MDP. \cite{marecki06_cmdp} consider continuous-time MDPs with finite
sets of states and actions. In contrast, our focus is on MDPs with
high-dimensional and uncountable state and action spaces. Recent work on
deep reinforcement learning (e.g.,
\cite{hausknecht16_iclr,mnih15_drl}) presents approaches for using
deep neural networks in conjunction with reinforcement learning to
solve short-horizon MDPs with continuous state spaces. These
approaches can be used as primitives in a complementary fashion with
task and motion planning algorithms, as illustrated in recent
promising work by \cite{wang18_active}.

Task planning efficiently computes solutions for complex goals. But, it can not handle manipulation problems with continuous domains that have an infinite branching factor. Though \emph{PDDL 2.1}~\cite{maria2003pddl2} allows using continuous variables, it still struggles to handle infinite branching factor.

\paragraph{\textbf{Motion Planning}} Recent research resulted in significant improvements in sampling-based motion planners. Probabilistic roadmaps (PRM)~\cite{kavraki1996probabilistic}  randomly sample from the C-space to generate a roadmap that can be lazily used to generate motion plans. Rapidly-exploring random trees (RRT)~\cite{Lavalle98rrt} computes a collision-free path from an initial robot configuration to the target configuration by connecting randomly sampled robot configurations from the C-space. Bi-directional RRT (BiRRT)~\cite{kuffner2000birrt} updates existing RRT to initiate search trees from the initial and goal configurations to boost the speed of motion planning. Constrained BiRRT (CBiRRT)~\cite{bernson2009cbirrt} extends the BiRRT technique constraining the search space by using projection techniques to explore configurations spaces and finds bridges between them. 

% While motion planning efficiently computes plans in high-dimensional continuous spaces, it lacks the capabilities to reason over a long horizon and to compute plans for complex tasks such as arranging a dining table. Our approach does not alter existing motion planning techniques and works with any off-the-shelf motion planner.

\paragraph{\textbf{Integrated Task and Motion Planning}}
\label{sec:itmp}
% Integrated task and motion planning (TMP) overcomes the limitations of task planning and motion planning by combining them. TMP utilizes the power of solving long horizon complex tasks and capabilities of motion planning to deal with possibly infinite state space. 

Most of the prior work in the field of integrated task and motion planning has focused on solving deterministic task and motion planning problems. Most of these approaches can be classified into three categories: \emph{1)} approaches that use symbols to guide the low-level motion planning, \emph{2)} approaches that extend high-level representations to simultaneously search high-level plans along with continuous parameters, and \emph{3)} approaches that use interleaved search for valid high-level plans with low-level refinements for its actions. Our approach falls under the last category. \citet{garrett2021integrated} present an exhaustive survey of these approaches; we discuss only the most closely related approaches here.

\paragraph{Approaches that use symbols to guide the motion planning:}
\cite{cambon09_asymov} introduced one of the earliest approaches named aSyMov. ASyMov uses symbolic knowledge to guide planning in geometric space using location references. \cite{plaku10_sampling} use a similar approach to allow combined task and motion planning for robots with constrained manipulators. Such approaches employ task planning as a heuristic for planning in the C-space, which may not always be efficient due to a lack of knowledge of geometric constraints at the task-planning level. In order to overcome this limitation, we interleave the process of computing motion plans and updating the high-level specification.

\paragraph{Approaches that extend high-level representations:} Another class of approaches ~\cite{dornhege12_semantic,garrett15_ffrob,garrett2020pddlstream} extends the high-level representation to allow the high-level planner to validate preconditions of the high-level actions in the geometric space while computing the high-level plan. \cite{dornhege12_semantic} do so by developing \emph{semantic attachments} for \emph{PDDL} representation that check the validity of each high-level action using a motion planner in the low level. \emph{FFRob}~\cite{garrett15_ffrob} uses pre-sampled robot configurations to discretize the problem and build a \emph{roadmap} to evaluate the preconditions of the high-level action. \emph{PDDLStream} \cite{garrett2020pddlstream} uses optimistic samplers to sample continuous arguments in the PDDL descriptions. Their optimistic samplers are analogous to ``generators'' used by our approach (explained later in Sec.~\ref{sec:entity}) that are used to instantiate abstract actions and serves the same purpose. Our approach and PDDLStream use these samplers to sample concrete values for symbolic abstract arguments.

% While these approaches require a carefully crafted high-level domain description including an explicit specification of the geometric constraints for all high-level actions, our approach does not require such explicit constraint specification for each action. These approaches are computationally expensive compared to our approach as the motion planner is invoked much more frequently as part of high-level planning.}

 \paragraph{Approaches that perform an interleaved search: } The last group of approaches performs an interleaved search to find a high-level solution that also has valid motion planning refinements in the low level. These approaches incrementally update the high-level models using the feedback from the low level while searching for the refinements. \cite{srivastava14_tmp} implement a modular approach that uses a planner-independent interface layer to allow communication between a task planner and a motion planner. 
%  The interface layer is used to compute refinements for high-level actions as well as update high-level models with the feedback received from the low-level environment while computing these refinements.
 \cite{dantam2018incremental} develop a constraint-based approach that incrementally adds constraints to the high-level specification of the problem discovered while trying to refine a high-level plan generated using an \emph{SMT}-based planner. Because these approaches commit to a single high-level model, it is not clear how they would be able to avoid dead ends. Additionally, all these approaches work only for deterministic problems and do not handle stochastic settings.
%  On the other hand, our approach maintains multiple abstract models and a defined strategy to explore them. This provides more thorough guarantees of probabilistic completeness.

To the best of our knowledge, the only approaches designed to handle stochastic task and motion planning problems were presented by \cite{kaelbling2011hierarchical}, \cite{hadfield15_modular}, and \citet{garrett2020online}. These approaches consider a partially observable formulation of the problem. \citet{kaelbling2011hierarchical} utilize regression modules on belief fluents to develop a regression-based solution algorithm. 
% While they address the more general class of partially observable problems, their approach follows a process of online, incremental discretization and does not address the computation of branching policies, which is the focus of this paper. 
\citet{hadfield15_modular} extend the work on deterministic task and motion planning by \citet{srivastava14_tmp} for partially observable settings. They use maximum likelihood observations~\cite{platt2010belief} to obtain a determinized high-level representation. 
% Our approach, in contrast to it, our approach uses abstraction to derive a high-level representation that is stochastic and requires a contingent solution. 
\citet{garrett2020online} develop an online algorithm that uses observational actions to gather belief about partially-observable environments and performs task and motion planning using discretized actions. These approaches address a more general class of partially observable problems. However, they do not address the computation of branching policies, which is the key focus of this paper.

% ~\newpage ~ \newpage

\section{Formal Framework}
\label{sec:formal}

% \subsection{Stoachstic Task and Motion Planning Problem}
% \subsection{Problem Statement}

% The main contribution of the paper is a probabilistically complete approach that computes task and motion policies for stochastic task and motion planning problems. We define the stochastic task and motion planning problem (STAMPP) as follows: 
% \begin{definition}
%     A stochastic task and motion planning problem (STAMPP) is defined as triplet $\langle \M, \alpha, \abs{\M}   \rangle$ where $\M$ is a low-level continuous stochastic shortest path (SSP) problem, $\alpha$ is an abstraction function, and $\abs{\M}$ is an abstract stochastic shortest path problem computed by applying the abstraction function $\alpha$ on the low-level SSP problem $\M$. 
% \end{definition}

% A solution of stochastic task and motion planning problem is a policy with actions from the concrete low-level continuous $SSP$ $\M$. Now, we define each component in the STAMPP triplet.

\subsection{Stochastic Task and Motion Planning Problem}
\label{sec:definition}

The main contribution of the paper is a probabilistically complete approach that computes task and motion policies for stochastic task and motion planning problems. We define the stochastic task and motion planning (STAMP) problem  as follows: 
\begin{definition}
    A stochastic task and motion planning (STAMP) problem  is defined as triplet $\langle \M, \alpha, \abs{\M}   \rangle$ where $\M$ is a low-level continuous stochastic shortest path (SSP) problem with $|\Am_\text{mp}| > 0$, $\alpha$ is an abstraction function, and $\abs{\M}$ is an abstract stochastic shortest path problem computed by applying the abstraction function $\alpha$ on the low-level SSP problem $\M$. 
\end{definition}

A solution for a stochastic task and motion planning (STAMP) problem  is a policy with actions from the concrete model $\M$. In this work, we consider solutions in the form of  a policy tree where each node $u_p$ in the tree represents a  state $s_{u_p}$ and an edge $e_p$ represents an action  $a_{e_p}$. The child of a node-edge pair $(u_p,e_p)$ in the policy tree refers to a possible outcome of executing the action $a_{e_p}$ at the state $s_{u_p}$. In the case of all deterministic actions, the tree would have a single branch. Now, we define the specific  \emph{entity abstraction} that we use to define the STAMP problem.

\subsection{Entity Abstraction}
\label{sec:entity}
In this paper, we use entity abstraction to define a stochastic task and motion planning problem. We define entity abstraction by extending the notion of abstractions introduced in Sec.~\ref{subsec:abstraction} as follows: Let $\mathcal{U}_l\;(\mathcal{U}_h)$ be the universe of $V_l\;(V_h)$ such that $|\mathcal{U}_h| \leq |\mathcal{U}_l|$. Let $\rho \; : \; \mathcal{U}_h \rightarrow 2^{\mathcal{U}_l}$ be a collection function that maps elements in $\mathcal{U}_h$ to the collection of $\mathcal{U}_l$ elements that they represent, e.g., $\rho(Table) = \{ loc \; : \; \land_{i}\; loc \cdot BoundaryVector_i < 0 \}$. Here $\rho$ binds Table $\in \mathcal{U}_h$ to a set of locations in $\mathcal{U}_l$ that are bounded by some polygonal boundary. Here $\mathcal{U}_l$ and $\mathcal{V}_l$ are low-level concrete universe and vocabulary and $\mathcal{U}_h$ and $\mathcal{V}_h$ are their abstract counterparts.

We define \emph{entity abstraction} $\alpha_\rho$ using the collection function $\rho$ as $\llbracket r \rrbracket _{\alpha_\rho(V_l)}(\tilde{o_1},\dots,\tilde{o_n}) = True $ iff $\exists\,o_1,\dots,o_n$ such that $o_i \in \rho(\tilde{o_i})$ and $\llbracket \psi_{r}^{\alpha_{\rho}}(o_1,\dots,o_n) \rrbracket _{S_l} = True$. We omit the subscript $\rho$ when it is clear from the context. Entity abstractions define the truth values of predicates over abstracted entities as disjunction of the corresponding concrete predicate instantiations. E.g., an object is in the abstract region ``kitchen'' if it is at one of the any locations in that region and an object is on ``table'' if it is at any location on the table-top. Such abstractions have been used for efficient generalized planning~\cite{srivastava2008AAAI} as well as answer set programming~\cite{zeynep18_aspocp}. These type of abstractions introduce terms that may not be identifiable at high level which makes these abstractions lossy and high-level models obtained by these abstractions inaccurate. E.g., the exact location of the table, the trajectory used to reach a configuration from current configuration.

Now, we use entity abstraction to define an abstract hybrid predicate for each hybrid predicate in our vocabulary by replacing each continuous argument in the hybrid predicate with its symbolic reference. E.g., \emph{$\abs{at}$($o_1$,$\overline{loc}$)} is an abstract hybrid predicate corresponding to a hybrid predicate \emph{at($o_1$,{loc})} where, \emph{$\overline{loc} \in \mathcal{U}$} is a symbolic reference of type $\tau_O$ for the continuous vector \emph{loc}. 

To formally define an abstract hybrid predicate, let $\alpha$ be a composition of entity abstraction and function abstraction. The abstract version of a concrete predicate $p_h$ is denoted as $\abs{p_h}_{\alpha}$. We omit the subscript $\alpha$ when it is clear from the context. We define $\abs{p_h}$ as follows:  
\begin{definition}
    A predicate $\abs{p_h}_\alpha(y_1,\dots,y_k,\bar{\theta}_1,\dots,\bar{\theta}_m)$ is an \textbf{abstract hybrid predicate} corresponding to a concrete hybrid predicate $p_h(y_1,\dots,y_k,\theta_1,\dots,\theta_m)$ iff all of its arguments $y_1,\dots,y_k,\bar{\theta}_1,\dots,\bar{\theta}_m$ are variables of type $\tau_O$ and $\forall\,\bar{\theta}_i \in \emph{arg}(\abs{p_h}_{\alpha})\, \theta_i \in \rho(\overline{\theta_i})$. $\abs{\mathcal{P}_h}_{\alpha}$ is a set of all abstract hybrid predicates.  
\end{definition}

We also define an abstract hybrid action for each hybrid action in the model using the abstraction $\alpha$. The abstraction $\alpha$ replaces each action argument of type $\tau_R$ with its symbolic reference of type $\tau_O$ and each concrete hybrid predicate in its precondition and effect with its abstract counterpart. Finally, we use these concepts to define an abstract SSP as follows:

\begin{definition}
    Given a concrete planning problem $\mathcal{M}$, an \textbf{abstract planning} problem $\abs{\mathcal{M}} = \langle \mathcal{O}, \abs{\mathcal{P}}, \abs{\mathcal{\Sm}}, \abs{\mathcal{A}}, T, C \abs{s_0}, \abs{S_g}, \gamma, H \rangle$, where, 

    \begin{itemize}
        \item $O$ is a set of names for the objects in the environment and symbolic references for entities in the environment, 
        \item $\abs{\mathcal{P}} = \mathcal{P}_{sym} \cup \abs{\mathcal{P}_h}$ is a set of abstract predicates,
        \item $\abs{\Sm}$ is a set of abstract states,
        \item $\abs{\mathcal{A}} = \mathcal{A}_{sym} \cup \abs{\mathcal{A}_{h}}$ is a set of abstract actions available to the robot, 
        \item $T: \abs{\Sm} \times \abs{\Am} \times \abs{\Sm} \rightarrow [0,1]$ is a transition function,
        \item $C: \abs{\Sm} \times \abs{\mathcal{A}} \rightarrow \mathbb{R}$ is a cost function,
        \item $\abs{s_0} \in \abs{\Sm}$ is the initial state,
        \item $\abs{S_g} \subset \abs{\Sm}$ is the set of goal states,
        \item $\gamma = 1$ is a discount factor (fixed),
        \item $H$ is a fixed horizon. 
    \end{itemize}

\end{definition}

A solution to an abstract planning problem is a valid sequence of actions $ \abs{\pi}_{\alpha} = \langle \abs{a_0},\dots, \abs{a_n} \rangle$ such that each action in $\abs{\pi}$, when applied sequentially from the initial state $\abs{s_0}$, the system reaches one of the goal states in $\abs{S_g}$.

\begin{figure}[t!]
  % \begin{small}
      \noindent \emph{Place($obj_1$, $config_1$, $config_2$, $target\_pose$, $traj_1$)} \\
      \begin{tabular}{rl}
          \emph{precon} & \emph{RobotAt($config_1$)} , \emph{holding($obj_1$)}, \\ 
                          & \emph{IsValidMP($traj_1$, $config_1$, $config_2$)}, \\
                          & \emph{IsCollisionFree($traj_1$)}, \\ 
                          & \emph{IsPlacementConfig($obj_1$,$config_2$,$target\_pose$)} \\ \\ 
          \emph{Concrete} & $\lnot$\emph{holding($obj_1$)}, \\
          \emph{effect}  & $\forall$ \emph{traj intersects(vol(obj, target\_pose))}, \\
                          & \emph{sweptVol(robot,traj)} $\rightarrow$ \emph{Collision($obj_1$,traj)}, \\
                          & \emph{RobotAt($config_2$)}, \emph{at($obj_1$,target\_pose)}  \\ \\ 
          \emph{Abstract} & $\lnot$ \emph{holding($obj_1$)}, \\
          \emph{effect}   & $\forall \; traj$ \? \emph{Collision($obj_1$,$traj_1$)}, \\
                          & $\lnot$\emph{RobotAt($config_1$)}, \emph{RobotAt($config_2$)}, \\
                          & \emph{at($obj_1$,target\_pose)} \\ 
      \end{tabular}
  % \end{small}
  \caption{Specification of concrete (above) and abstract(below) effects of a one-handed robot's action for placing an object}
  \label[fig]{abs_example1}
\end{figure}

Concretization operation is performed by replacing abstract symbolic references with concrete objects from their low-level domains. For instance, let $\Sm_h$ be the set of abstract states generated when an abstraction
$\alpha$ is applied on a set of concrete states $\Sm_l$. For any
$s_h \in \Sm_h$, the \emph{concretization function}
$\Gamma_\alpha(s_h) = \set{s_l \in \Sm_l: \alpha(s_l)=s_h}$ denotes the set of
concrete states represented by the abstract state $s$. Similarly, abstract hybrid actions are refined by grounding abstract entities using values from their low-level domains. But, generating the complete concretization of an
abstract state can be computationally intractable, especially in cases
where the concrete state space is continuous. In such situations, the
concretization operation can be implemented as a \emph{generator} that
incrementally samples elements from an abstract argument's concrete
domain. A generator can also be designed in way that it validates the generated values while generating them and only yield valid instantiations for the symbolic arguments.

\paragraph{\textbf{Example}}  Consider the specification of a robot's action of placing an item as a part of an SSP. In practice, low-level accurate models of such actions may be expressed as generative models or simulators. Fig. \ref{abs_example1} helps to identify the nature of abstract representations needed for expressing such actions. For readability, we use a convention where preconditions are comma-separated conjunctive lists and universal quantifiers represent conjunctions over the quantified variables. 

Fig. \ref{abs_example1} shows the specification of an action that places an object at the specified pose. Concrete description of the action requires action arguments representing object to be placed (\emph{obj}$_1$), the initial and final configuration of the robot (\emph{config}$_1$, \emph{config}$_2$), target pose for the object (\emph{target\_pose}), and the motion trajectory that takes the robot from its initial configuration to final configuration (\emph{traj}$_1$). Here \emph{obj$_1$} is an argument of type $\tau_O$ and \emph{config}$_1$, \emph{config}$_2$, \emph{target\_pose}, and \emph{traj}$_1$ are continuous 0arguments of type $\tau_R$. The abstract counterpart of this action is computed by replacing the continuous arguments of type $\tau_R$ in the concrete version with symbolic arguments representing abstract entities as mentioned earlier. E.g., \emph{target\_pose} in the abstract specification is a symbolic reference for all valid target poses for the object and \emph{traj}$_1$ is a reference for all valid motion trajectories that take the robot from \emph{config}$_1$ to \emph{config}$_2$. Values of these arguments can not be determined precisely in the abstracted space and thus a subset of preconditions and effects can not be evaluated while planning with abstract models.  E.g., it is not possible to determine whether a trajectory is collision-free as part of the precondition. Similarly, it is also not possible to determine what trajectories will be in a collision when an object is placed at a certain pose in the abstract model. Such predicates are annotated in the set of effects with the symbol \?. While computing abstractions in such a way loses important information, the abstract model is still sound~\cite{srivastava14_tmp,srivastava2016metaphysics}. 

\begin{figure}[t!]
  \begin{center}
    \includegraphics[width=0.5\columnwidth]{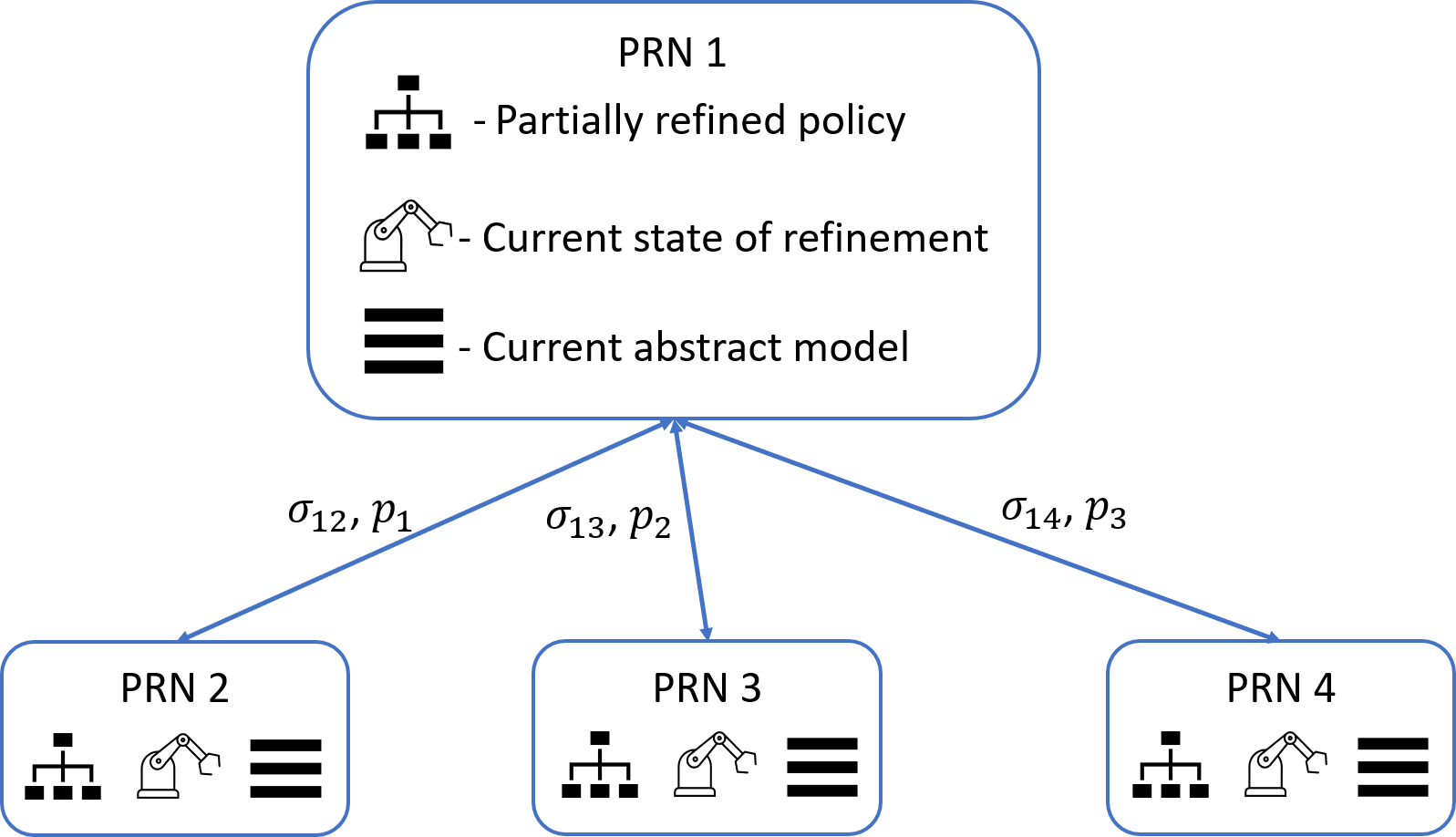}
  \end{center}
  \caption{Plan refinement graph (PRG) used to maintain
separate abstract models. Each plan refinement node (PRN) contains an abstract
model, partially refined policy, and current state of refinement. Each edge contains refinement for a partial policy ($\sigma_{ij}$) and a failure reason ($p_k$). }
  \label{fig:prg}
\end{figure}

\begin{algorithm}[t!]
  % \begin{small}
    \KwIn{model $\M$, abstraction function $\alpha$, concretization function $\gamma$,  abstract model $\abs{\M}_{\alpha}$, symbolic planner $P$}
    \KwOut{anytime, contingent policy that is executable in $\M$ }
    Initialize PRG with a node with an abstract policy $\abs{\pi}$ for $\mathcal{G}$
    computed using \emph{P}\;
    \While{solution of desired quality not found}
    {
      $u$ $\gets$ GetPRNode()\;
      $\abs{\M}_u$ $\gets$ GetAbstractModel($u$)\;
      $\abs{\pi}_u$ $\gets$ GetAbstractPolicy($\abs{\M}_u$, $\mathcal{G}$, $P$, $u$)\;
      Choice $\gets$ NDChoice\{\emph{RefinePolicy}, \emph{RefineAbstraction}\}\;
      \If{Choice = \emph{RefinePolicy}}{
        \While{$\abs{\pi}_u$ has an unrefined RTL path and resource limit is not
          reached}{
            $path$ $\gets$ GetUnrefinedRTLPath($\abs{\pi}_u$)\;
          \If{explore\tcp{non-deterministic}}{
            replace a suffix of refined partial $path$ with a random action\;
          }
          Search for a feasible concretization of $path$\;
        }
      }
      \If{Choice = \emph{RefineAbstraction}}{
        $path$ $\gets$ GetUnrefinedRTLPath($\abs{\pi}_u$)\;
        $\sigma \gets$ ConcretizeFirstUnrefinedAction($path$)\;
        failure\_reason $\gets$ GetFailedPrecondition($\sigma$
        )\;
        $\abs{\M'}$ $\gets$ UpdateAbstraction($\abs{\M}$, failure\_reason) \;
        % random\_action $\gets$ choose\_random\_action(${\cal D}$)\;
        % next\_state $\gets$ apply\_action(random\_action)\;
        $\abs{\pi'}$ $\gets$ merge($\abs{\pi}$, GetAbstractPolicy($\abs{\M'}$, $\mathcal{G}$, solver))\;
        generate\_new\_pr\_node($\abs{\pi'}$, $\abs{\M'}$)\;
      }
      recompute  $p/c$ ratio for unrefined RTL paths\;
    }
  % \end{small}
\caption{\small HPlan Algorithm}
\label{alg:atam}
\end{algorithm}

Refining (instantiating) the abstract \emph{place} action sampling concrete values for each symbolic abstract entity in its arguments from their low-level domain. E.g., refining the symbolic entity \emph{target\_pose} would require using a generative model such as a simulator to sample a valid pose for the object being placed and computing a valid motion plan that takes the robot from its current configuration to a configuration that places the object at the sampled pose. This can be implemented using a backtracking search that tries to instantiate abstract entities in a sequential order while evaluating concrete preconditions for the instantiations.

\section{Computing Task and Motion Policies}
\label{sec:algo}

\subsection{HPlan Algorithm}
\label{sec:atm}
We extend the idea of planning with abstractions briefly discussed by \cite{srivastava2016metaphysics} to perform task and motion planning in stochastic environments using abstraction hierarchies. The goal is to find a valid high-level policy that also has valid low-level refinements for each of its actions. We propose the HPlan algorithm (Alg. \ref{alg:atam}) that performs hierarchical planning with arbitrary abstraction and concretization function.

% In this section, we describe our approach for computing task and motion policies as defined above. Remember that every abstract action $[a] \in [\mathcal{M}]$  (e.g., \emph{Place($obj_1$,$config_1$,$config_2$,$pose_1$,$traj_1$)}) has symbolic arguments that has to be instantiated to generate a concrete action $a$, that can be executed in the concrete model $\mathcal{M}$. The goal is to compute instantiations for each action in the high-level solution. 

HPlan (Alg.~\ref{alg:atam}) uses a policy refinement graph (PRG) to keep track of different abstract models and their corresponding policies. As shown in Fig.~\ref{fig:prg}, each node $u$ in a PRG contains an abstract model $\abs{\M}_{u}$, an abstract policy $\abs{\pi}_u$, and the current state of refinement for each action $\abs{a_j} \in \abs{\pi}_u$. An edge $(u,v)$ in a PRG from a node $u$ to a node $v$ consists of a partial refinement of the policy ($\sigma_{uv}$) and a failed precondition of the first action from $\abs{\pi}_u$ that does not have a valid motion planning refinement. Our approach combines two processes: $1)$ Concretizing the abstract policy, and $2)$ refining the abstract model.

HPlan (Alg.~\ref{alg:atam}) performs the above-mentioned two steps in an interleaved manner. The algorithm starts by 
% a single policy refinement node (PRN) in the PRG with an initially provided abstract model. Line $1$ 
initializing the PRG with a node containing this abstract model $\abs{\M}$, and an abstract policy $\abs{\pi}$ computed using an off-the-shelf symbolic solver that achieves the goal $\mathcal{G}$ (line $1$). Each iteration of the main loop (line $2$) selects a policy refinement node (PRN) $u$ from the PRG using a defined strategy (line $3$). 
%  and extracts a root-to-leaf (RTL) path from the current PRG node's policy $\abs{\pi}_u$ such that the path has at least one action that has not been instantiated (line $3$-$5$). 
 Arbitrary strategies can be used to make this selection. HPlan uses an off-the-shelf task planner to compute a high-level policy for the current abstract model if the selected PRN does not already have a high-level policy (line $5$). Once a policy is computed (or obtained), HPlan non-deterministically decides (line $6$) to either refine the high-level policy in the selected PRN by instantiating abstract arguments  of actions in the policy (lines $7$-$13$) or to update the high-level abstractions to compute accurate high-level policies (lines $14$-$20$). The algorithm carries out these interleaved steps in as follows: 

\paragraph{\textbf{a) Concretizing the Abstract Policy}}  
Lines \emph{8-13} search for a valid concretization (refinement) of the high-level policy selected/computed on line $5$ by concretizing the abstract actions with actions from the concrete domain $\mathcal{M}$ using the concretization function $\Tau_{\alpha}$ as explained in Sec.~\ref{sec:entity}. To refine a high-level policy, a root-to-leaf (RTL) path is selected that has at least one action that has not been refined. Each unrefined action is concretized using a local backtracking search (line $13$)~\cite{srivastava14_tmp}. A concretization $c_0, a_1, c_1, \ldots, a_k, c_k$ is a valid concretization of an RTL path $\abs{s_0}, \abs{a_1}, \abs{s_1}, \ldots, \abs{a_k}, \abs{s_k}$ is valid iff $c_{i+1} \in a_{i+1}(c_i)$ and $c_i \models precon(a_i+1)$ for $i = 0, \ldots, k -1$. A policy is refined when concretization for each action in every RTL path in the policy is computed. However, due to lossy nature of the abstraction, it may be possible that no valid concretization exists for the policy $\abs{\pi}_u$. For example, consider an abstraction which drops \emph{InCollision} predicate that checks whether a trajectory is in collision with some object or not from an action that places an object at a desired pose. Such high-level actions would not have any valid concretization if all the trajectories are being obstructed by some object in the low level.

\paragraph{\textbf{b) Refining the Abstract Model}} 
Lines \emph{15-20} fix a concretization for the partially refined policy selected on line $5$ and identify the earliest abstract state in the selected policy whose subsequent action's concretization is infeasible. The abstract model is refined by adding the true form of the violated precondition at the low level. Continuing the same example, if all the trajectories from the current state to the state that has the object at the desired pose are in a collision with some other object $obj_x$, the concrete precondition \emph{InCollision(traj, $obj_x$)} is violated at the concrete level and is added to the current abstract model. The rest of the policy after this abstract state is discarded.  Lines \emph{19-20} use the new model to compute a new policy. The symbolic planner is invoked to compute a new policy from the updated state; its solution policy is unrolled as a tree of bounded depth and appended to the partially refined path. This allows the time horizon of the policy to be increased dynamically.

   \begin{theorem}
    \label{thm:pc}
    If there exists a proper policy that reaches the goal within
    horizon $h$ -{}- i.e. the probability of reaching the goal is $1.0$ -{}- and has feasible
    low-level concretization for each of its actions,  and measure of these refinements under the probability density of the generators is non-zero, then Alg.~\ref{alg:atam} will find it with
    probability $1.0$ in the limit of infinite samples.
  \end{theorem}

  \begin{proof}
    Let $\pi_p$ be the proper policy that achieves the goal with horizon $h$ and has valid low-level concretization for each of its actions. Consider a policy $\pi_i$ inside a PRN $i$ at an intermediate step of Alg.~\ref{alg:atam}; let $k$ denote the minimum depth up to which $\pi_p$ and $\pi_i$ match. Here, $k$ denotes a measure of correctness. When PRN $i$ is selected for refinement, eventually Alg.~\ref{alg:atam} would try to compute low-level concretization for an action at depth $k+1$ that does not match with the proper policy $\pi_p$. In this case, there is a chance that Alg.~\ref{alg:atam} would select the correct action (that matches with $\pi_p$ at depth $k+1$) under the \emph{explore} condition (lines $10$-$12$) of Alg.~\ref{alg:atam} and then generates a plan that reaches the goal state. Finite number of discrete actions in the abstract model and the fixed horizon ensures that in time bounded in expectation, HPLan will generate a policy with the measure of correctness $k+1$ and eventually with the measure of correctness $h$. Once the algorithm finds the policy with the measure of correctness $h$, it stores it in the PRG and is guaranteed to find feasible refinements with probability one if the measure of these refinements under the probability-density of the generators is non-zero.
  \end{proof}
  
  % \begin{proof} 
  %   \textcolor{red}{\emph{(Sketch)} Let $\pi_p$ be a proper policy that reaches the goal withing horizon. Consider a policy
  %   $\pi$ in a \emph{PRG}; let $k$ denote the minimum depth up to which
  %   $\pi_p$ and $\pi$ match. $k$ will be used as a \emph{measure of correctness}. When the plan refinement node containing $\pi$ is selected, suppose we try to refine one of the child nodes of depth $k+1$ in the partial path that had the $k$-length prefix consistent with the correct solution $\pi_p$.  The algorithm selects the correct child action with non-zero probability under the \emph{explore} step (line $11$) and then generates a plan to reach the goal from the resultant state. The finite number of discrete actions and the fixed horizon ensures that in time bounded in expectation, \emph{HPlan} will generate a policy with the measure of correctness $k+1$. Once the algorithm finds the policy with the measure of correctness \emph{h}, it stores it in the PRG and is guaranteed to find feasible refinements with probability one if the measure of these refinements under the probability-density of the generators is non-zero.}
  % \end{proof}

  \begin{figure}[t!]
    \begin{center}
      \includegraphics[width=0.6\columnwidth]{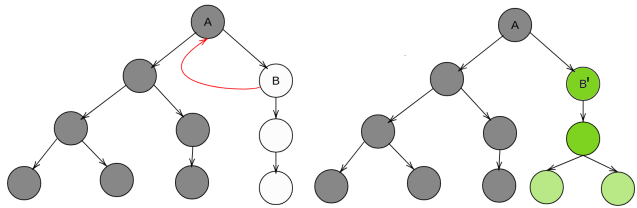}
    \end{center}
    \caption{Left: Backtracking from node $B$ invalidates the concretization of subtree rooted at $A$. Right: Replanning from node $B$}
    \label{fig:ptree}
  \end{figure}

\subsection{HPlan for STAMP}
We enhance the basic Alg.~\ref{alg:atam} in two primary directions to facilitate STAMP problems. These optimizations allow Alg.~\ref{alg:atam} to compute anytime solutions for STAMP problems and improve the search of concretization of abstract policies.

\paragraph{\textbf{Search for Concretizations}} 
Sampling-based backtracking search performed by Alg.~\ref{alg:atam} (line $13$) to concretize the abstract actions suffers from a few limitations in stochastic settings that are not present in the deterministic settings. Fig.~\ref{fig:ptree} illustrates the problem. The gray nodes in the image show the actions which are concretized. White nodes represent actions that are yet to be concretized. Sibling nodes represent the non-deterministic action outcomes. Now, if the action in the node $B$ does not accept any valid concretization, backtracking to node $A$ and changing its action's concretization would invalidate concretization for the entire subtree rooted at node $A$. Alg.~\ref{alg:atam} handles such scenarios by non-deterministically selecting whether to perform backtracking searching or not (line $6$) and by maintaining different abstract models through \emph{PRG} and employing a resource limit (line $8$) to explore them simultaneously.

\paragraph{\textbf{Anytime Computation for Task and Motion Policies}} 
The main computational challenge for Alg.~\ref{alg:atam} in stochastic settings is that the number
of root-to-leaf (RTL) branches grows exponentially with the time
horizon and the number of contingencies in the domain. In most scenarios, not all contingencies are equally probable. Each RTL path has a certain
probability of being encountered; refining it incurs a computational
cost. Waiting for a complete refinement of the policy tree results in wasting a lot of
time as most of the situations have a very low probability of being encountered.  The optimal selection of the paths to refine within a fixed
computational budget can be reduced to the knapsack
problem. Unfortunately, we do not know the precise
computational costs required to refine an RTL path. However, we can approximate this cost depending on the number of actions in an RTL path and the size of the domains of the arguments of those actions. Furthermore, the knapsack problem is NP-hard. However, we can compute provably good approximate solutions to this problem using a greedy approach: we prioritize the selection of a path to refine based on the probability
of encountering that path \emph{p} and the estimated cost of
refining that path \emph{c}. We compute $p/c$ ratio for all the paths and select the unrefined path with the largest ratio for refinement (line $9$ and $15$). $p/c$ ratio for each path is updated after each iteration of the main loop (line $21$). Intuitively, our approach works as follows:

\begin{figure*}[t]
  \centering
  \includegraphics[width=0.8\textwidth]{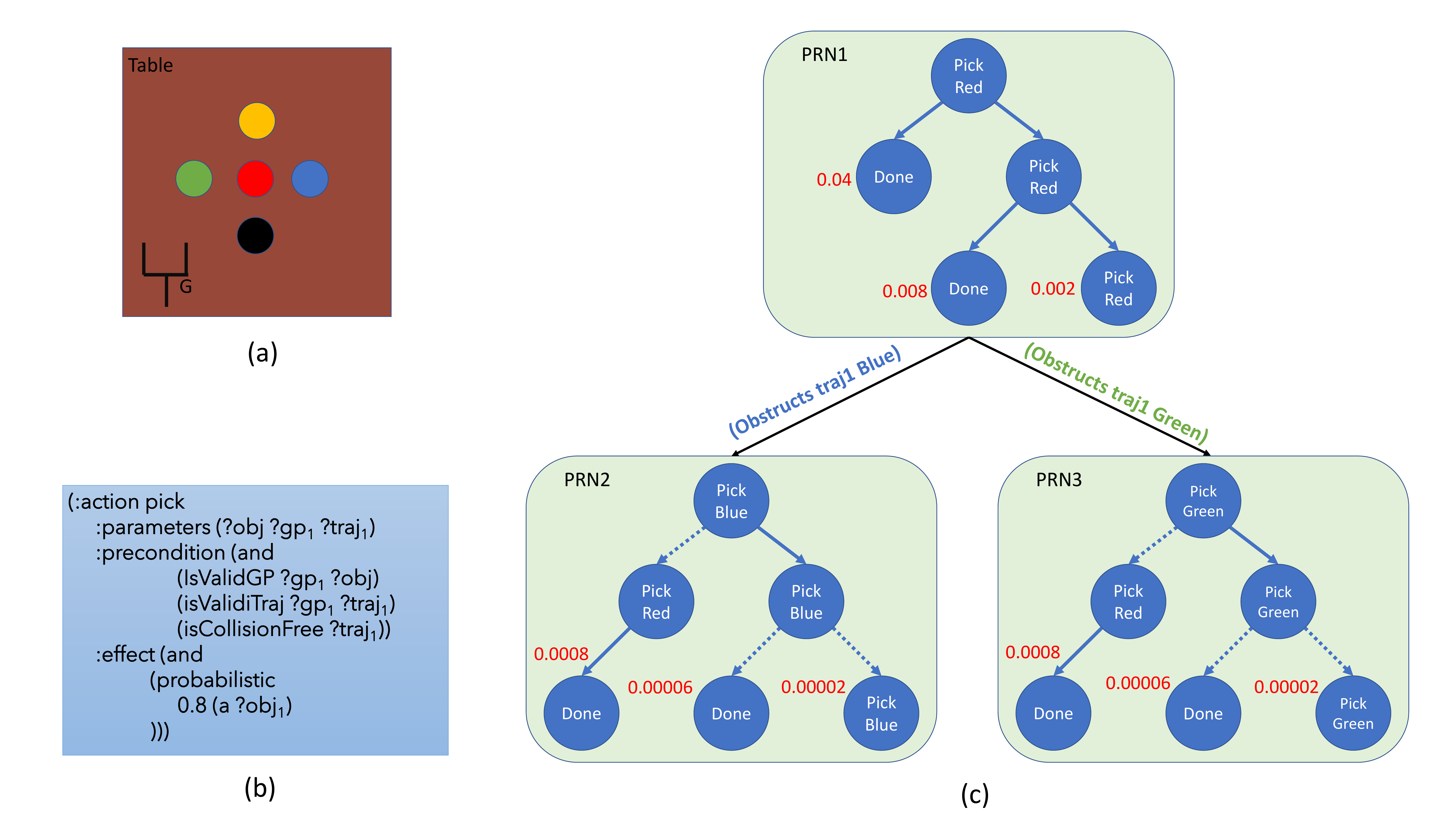}
  \caption{A working example for Alg.~\ref{alg:atam}. (a) shows initial environment configuration. Goal for the robot is to pick up the ``Red'' object which is surrounded by ``Blue'', ``Green'', ``Orange'', and ``Black'' objects. G is the end-effector of a robot. (b) shows a high-level, abstract task specification of the ``pick'' action. (c) shows the policy refinement graph (PRG) which is generated incrementally by Alg.~\ref{alg:atam}. Each green box represents a policy refinement node (PRN). Tree in each PRN represents a high-level policy. Each node in a high-level policy is a state-action pair. For brevity, we only show high-level action in the node. Trees with dotted lines are partial policies. Red number represents $p/c$ ratio for each RTL path in a policy. }
  \label{fig:working_example}
  
\end{figure*}

\begin{figure*}[t!]
  \begin{subfigure}{\textwidth}
    \centering
    \vspace{1em}
      \includegraphics[width=1.5in,height=1.5in]{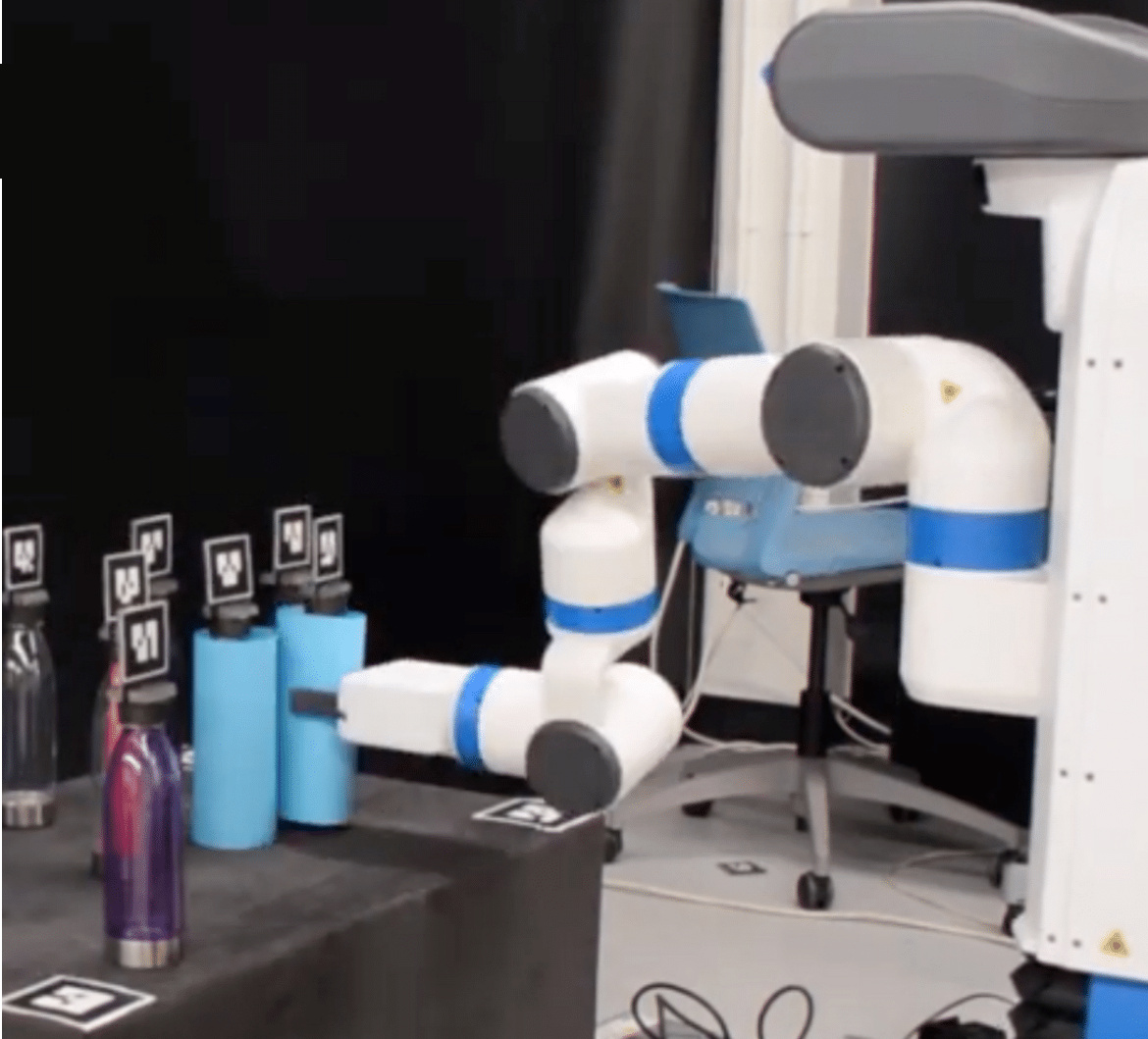}
      \includegraphics[width=1.5in,height=1.5in]{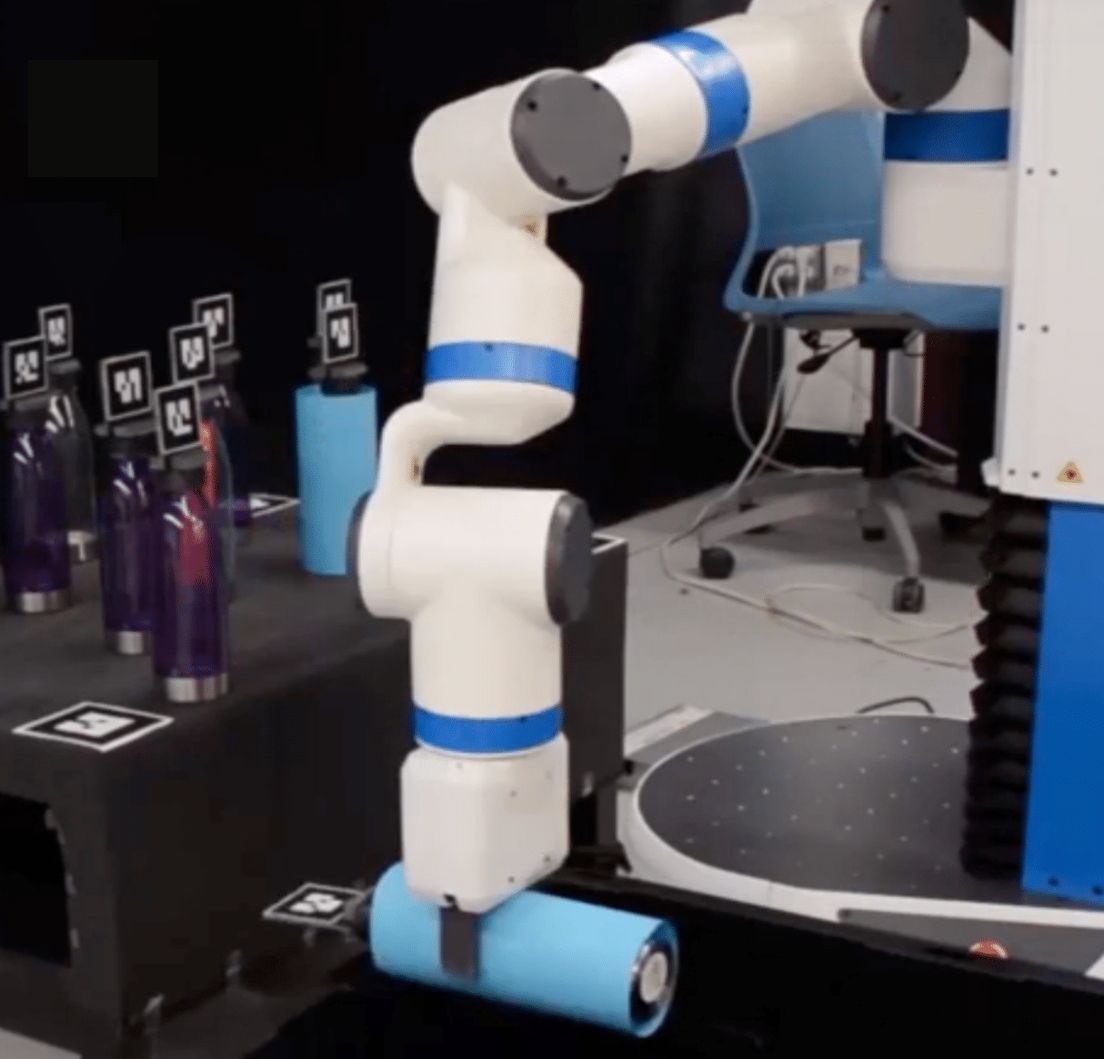}
      \includegraphics[width=1.5in,height=1.5in]{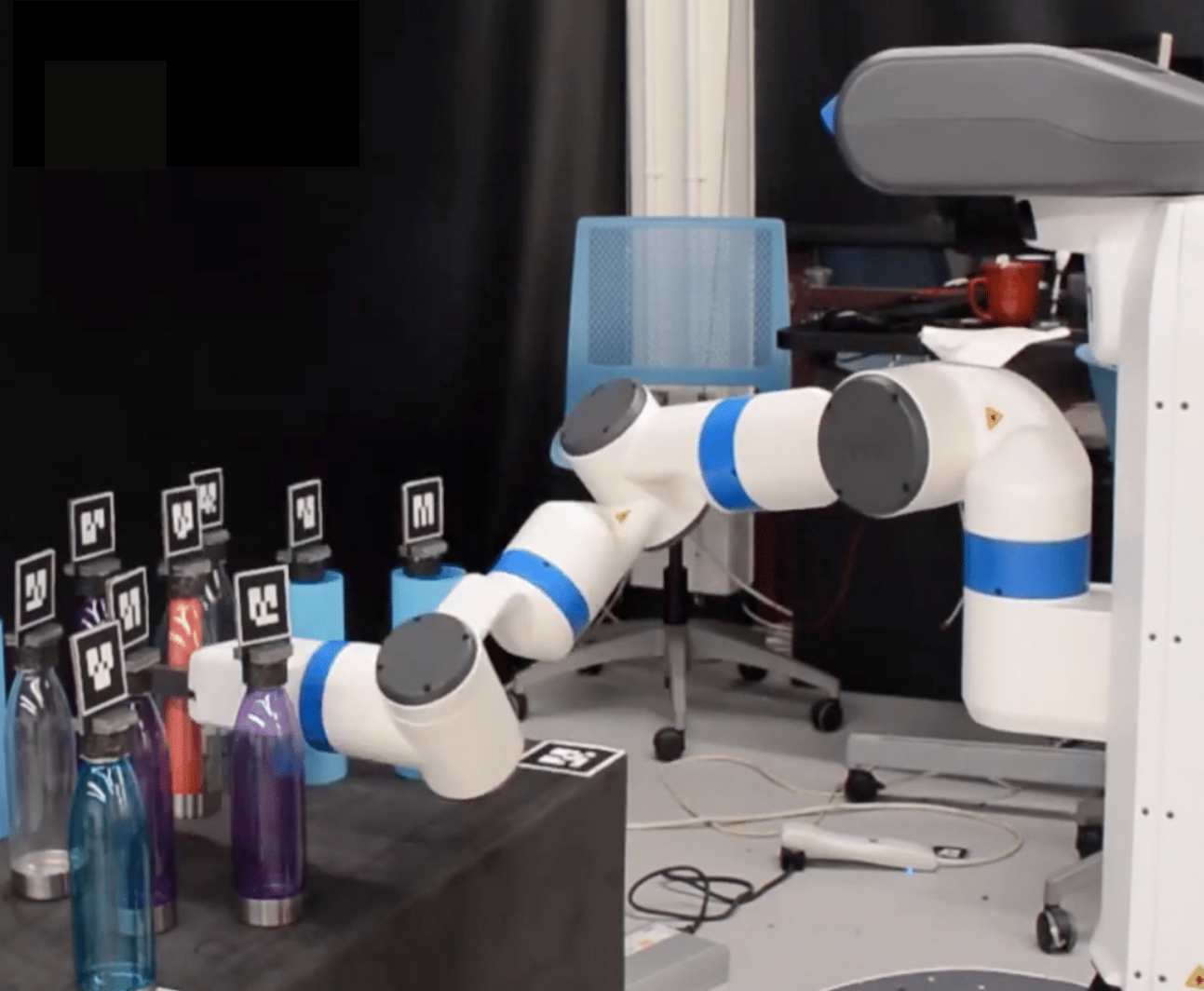}
      % \caption{\footnotesize The Fetch mobile manipulator uses a STAMP policy to
      %   pickup a target bottle while avoiding those that are likely to
      %   be crushed. It replaces a bottle that wasn't crushed (left), 
      %   discards a bottle that was crushed (center) and picks up the target
      %   bottle (right).  }
    \end{subfigure}
    \begin{subfigure}{\textwidth}
      \centering
    \includegraphics[width=1.5in,height=1.5in]{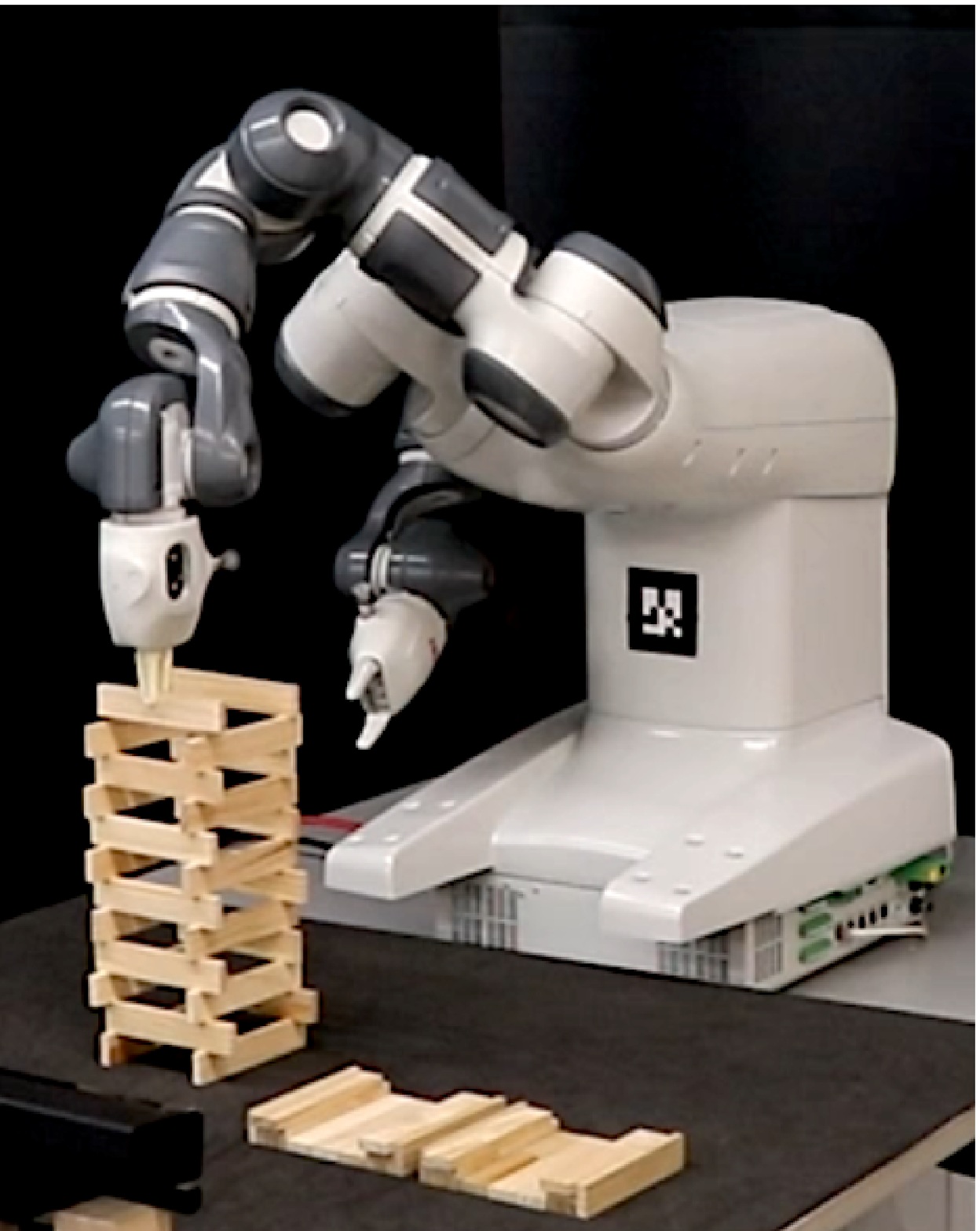}
    \includegraphics[width=1.5in,height=1.5in]{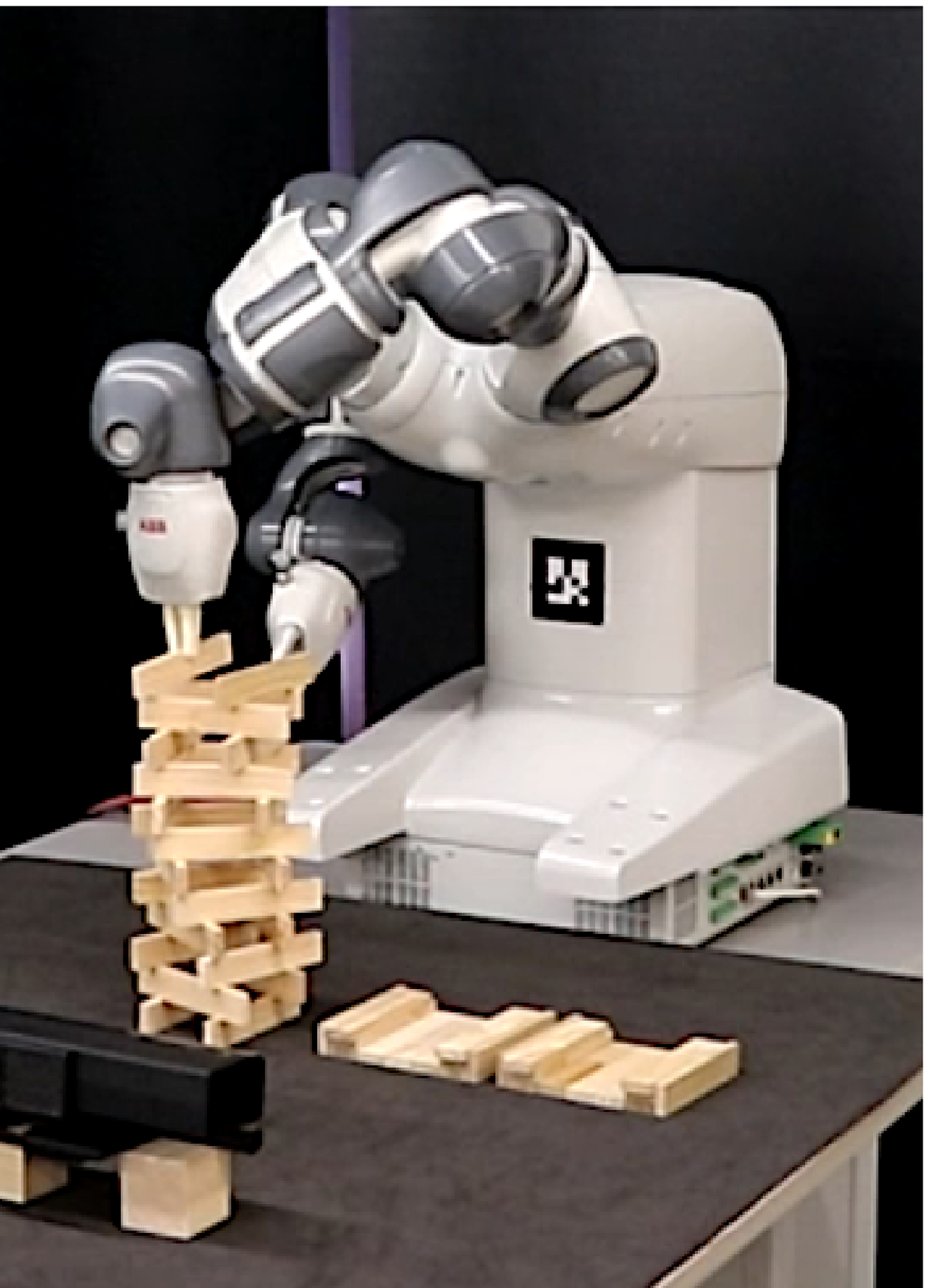}
    \includegraphics[width=1.5in,height=1.5in]{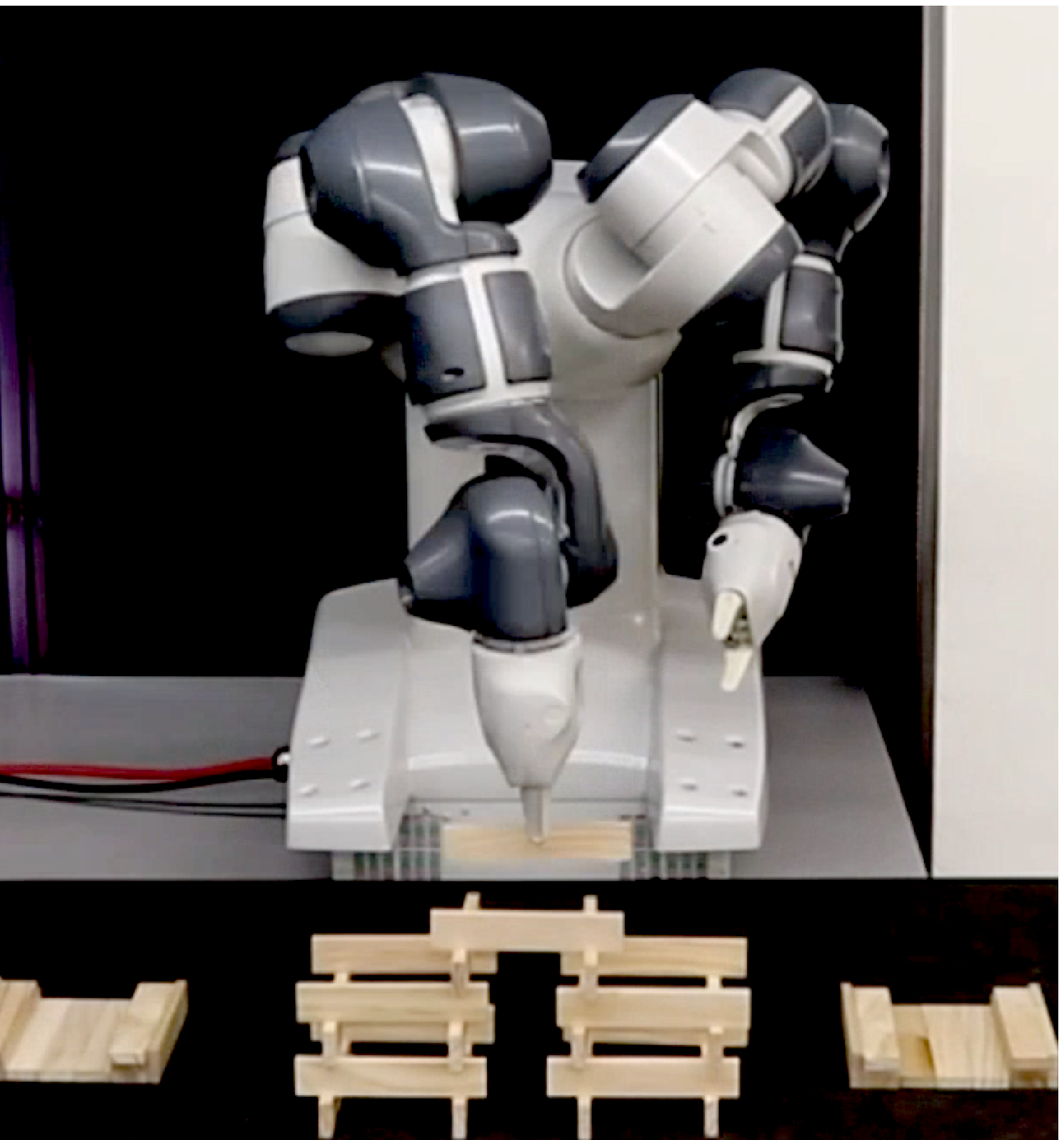}
    \end{subfigure}  \\
    \caption{
    Top: Cluttered Table: The Fetch mobile manipulator uses a STAMP policy to
      pick up a target bottle while avoiding those that are likely to
      be crushed. It replaces a bottle that wasn't crushed (left), 
      discards a bottle that was crushed (center) and picks up the target
      bottle (right).    
    Bottom: Building Structures with Keva Planks: ABB YuMi builds Keva structures using a STAMP policy:
        12-level tower (left), twisted 12-level tower (center), and
        $3$-towers (right).}
    \label{fig:exp_1_2}
  \end{figure*}
  
  \begin{figure*}

    \centering
    \includegraphics[height=0.24\columnwidth]{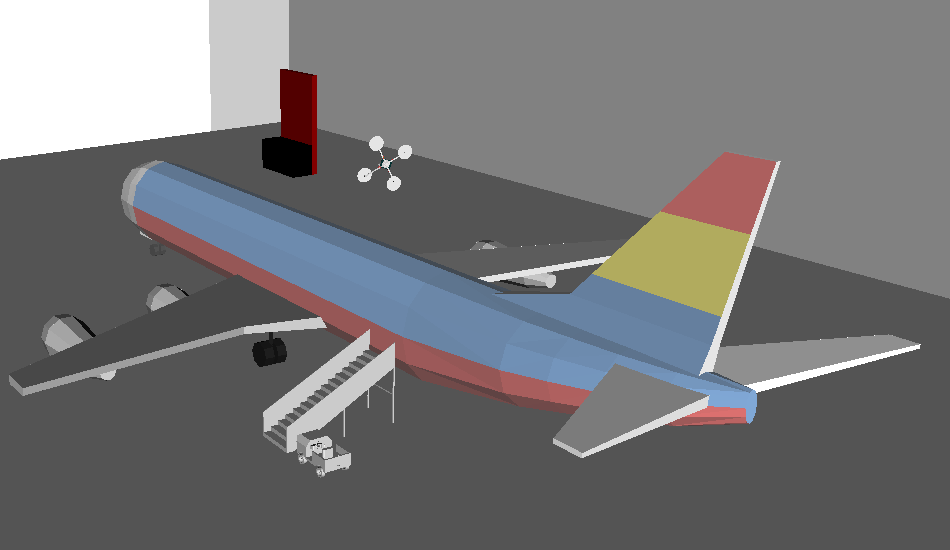}
    \hspace{-0.5em}
    \includegraphics[height=0.24\columnwidth]{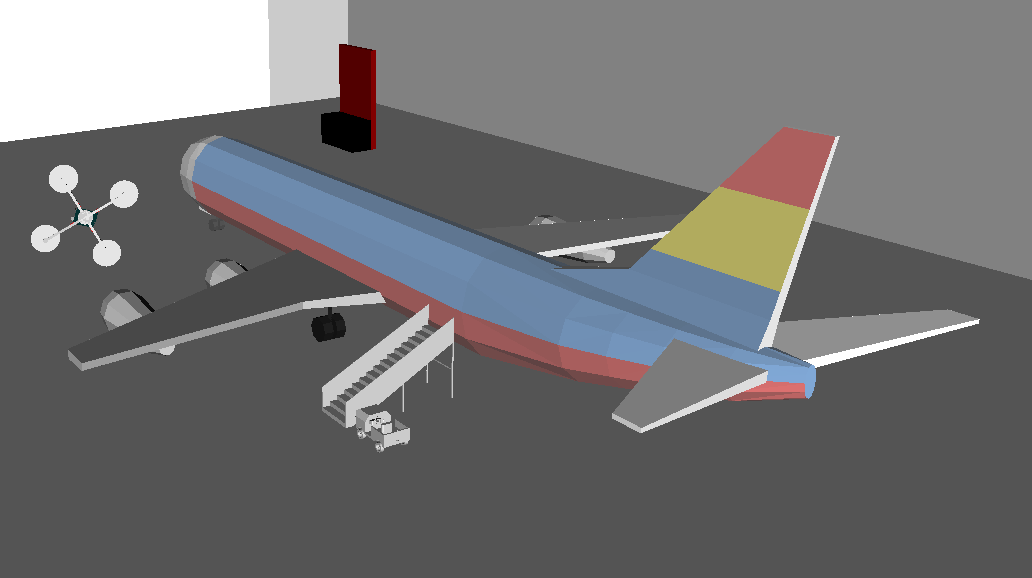}
    \hspace{0.3em}
    \includegraphics[height=0.24\columnwidth]{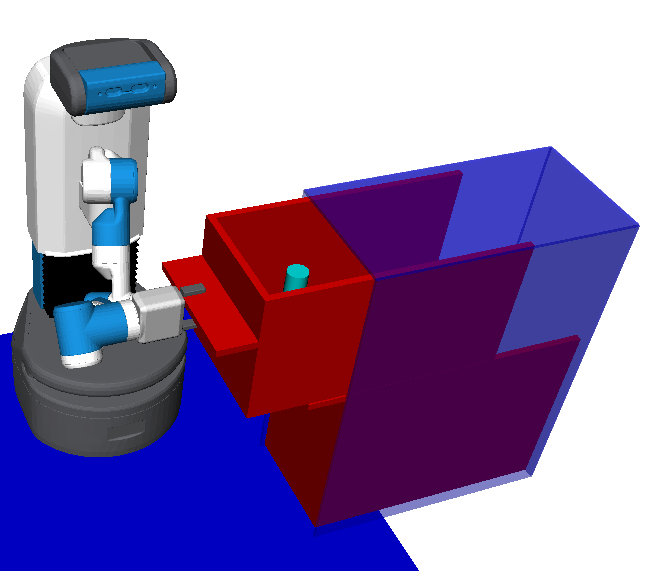}
    \hspace{-0.7em}
    \includegraphics[height=0.24\columnwidth]{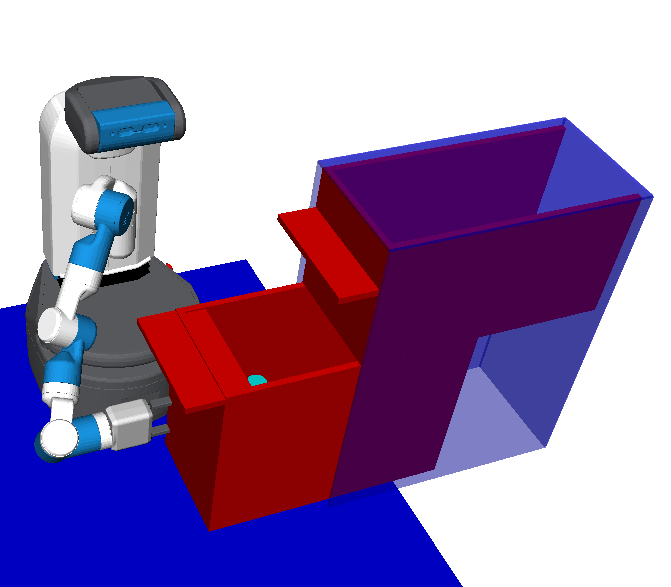}

  % \begin{subfigure}{0.45\textwidth}
  %   \centering
  %   \includegraphics[width=1.5in]{./hanger_1.png}
  %   % \hspace{1em}
  %   \includegraphics[width=1.5in]{./hangar_2.png}
   
  % \end{subfigure}
  % \begin{subfigure}{0.45\textwidth}
  %   \centering
  %   \includegraphics[width=1in]{./drawer_1_new.png}
  %   % \hspace{1em}
  %   \includegraphics[width=1in]{./drawer_2_new.png}
  % \end{subfigure} \\
  \caption{Top: Aircraft Inspection: UAV inspects faulty parts of an aircraft in an airplane hangar and alerts the human about the location of the fault. UAV's movements and sensors are noisy, so it may drift from its location or fail to locate the fault. Bottom: Find the can: Fetch searches for a can in drawers. The can can be placed in one of the drawers stochastically.}
  \label{fig:exp_3_6}
\end{figure*}

\paragraph{\textbf{Example}} Fig.~\ref{fig:working_example} illustrates our approach for solving a STAMP problem using Alg.~\ref{alg:atam}. Fig.~\ref{fig:working_example}(a) shows a low-level configuration of an environment. Here, a robot with an effector G is asked to pick up the red object which is surrounded by green, blue, orange, and black objects. Fig.~\ref{fig:working_example}(b) shows a high-level specification of the pick action in the PPDDL format. Fig.~\ref{fig:working_example}(c) shows the policy refinement graph (PRG) that is generated incrementally by Alg.~\ref{alg:atam}.

As explained in Sec.~\ref{sec:atm}, Alg.~\ref{alg:atam} starts with a single node in the PRG -{}- in this case, PRN$1$. Initially, PRN$1$ does not have a high-level policy. Alg.~\ref{alg:atam} uses the abstract action descriptions (abstract model $\abs{\M}$) and an off-the-shelf high-level SSP solver to compute a high-level symbolic policy that reaches the abstract goal (line $5$) and computes $p/c$ ratios for each RTL path in this abstract policy. To compute this ratio, we estimate the cost of refining each high-level action as follows: Suppose that the generators used to concretize the pick actions samples four grasp up poses in four cardinal directions to pick up the object and five motion planning trajectories between robot's current configuration to the grasp pose, then the approximate cost of refining this action would $4\times5 = 20$. We use this approximate cost to compute $p/c$ ratios (red numbers in Fig.~\ref{fig:working_example}). The next step for Alg.~\ref{alg:atam} is to non-deterministically decide between refining the computed high-level policy and refining the abstraction. 

Assume Alg.~\ref{alg:atam}  non-deterministically decides to refine the high-level policy (line $6$). After deciding to refine the high-level policy, Alg.~\ref{alg:atam} selects an RTL path using the $p/c$ ratio and tries to refine each action on this path by instantiating each symbolic argument. Here in this example, the first RTL path would only have a single high-level action \emph{pick(Red, gp$_1$, traj$_1$)} that needs refinement. To instantiate the high-level pick action, it first uses a generator to sample one of the possible grasp poses for the red object and then uses a low-level motion planner to generate a trajectory that would take the robot end-effector G to the selected grasp pose from its current pose. As the red object is surrounded by other objects, all the trajectories that take the end-effector to the grasp pose, are in collision with at least one object. This violates the precondition of the pick action making the refinement infeasible. Alg~\ref{alg:atam} continues trying to refine this action using the local and global backtracking search for a fixed amount of time before again making a non-deterministic choice between refining the high-level policy or the high-level abstraction.

Suppose this time Alg.~\ref{alg:atam} decides to refine the high-level abstraction. To do so, it would identify the failing precondition preventing a valid refinement for the high-level policy and generate a set of child nodes in the PRG -{}- PRN$2$ and PRN$3$ in this case corresponding to failing preconditions \emph{Obstructs(traj$_1$, Blue)} and \emph{Obstructs(traj$_1$, Green)}. Once these nodes are generated, Alg.~\ref{alg:atam} would move on to the next iteration of the approach where it would select one of these newly generated plan refinement nodes and repeat the entire process until a complete task and motion policy is computed. 

\begin{theorem} \label{thm:knapsack} Let $t$ be the time since the start of the algorithm
  at which the refinement of any \emph{RTL} path is completed. If
  path costs are accurate and constant then the total probability of
  unrefined paths at time $t$ is at most $1 - opt(t)/2$, where
  $opt(t)$ is the best possible refinement (in terms of the
  probability of outcomes covered) that could have been achieved in
  time $t$.
\end{theorem}

% \begin{figure*}
  % \begin{subfigure}{\textwidth}
  %   \centering
  %     % \includegraphics[width=3in]{./clutter_1_1.png} 
  %     \includegraphics[width=3in]{./sort_top_1.eps} 
  %     % \hspace{1em}
  %   %  \hspace{1em} \includegraphics[width=3in]{./clutter_2_1.png}
  %    \hspace{1em} \includegraphics[width=3in]{./sort_top_2.eps}
  %   \caption{Fetch sorts a cluttered table. All the blue cans have to be placed on the left table and all the green cans have to be placed on right table. Red cans act as obstacles. Left: The initial state for a problem. Right: The goal state.}
  %   \label{fig:sort}
  %   \end{subfigure} 

\begin{proof}
  (\emph{Sketch})
   The proof follows from the fact that the greedy algorithm achieves a
   2-approximation for the knapsack problem. In practice, we estimate the
   cost as $\hat{c}$, the product of measures of the true domains of each
   symbolic argument in the given \emph{RTL}. Since, $\hat{c}\ge c$ modulo
   constant factors, the priority queue never can only underestimate the relative value of refining a path, and the algorithm's coverage of
   high-probability contingencies will be closer to optimal than the
   bound suggested in the theorem above. This optimization gives a user
   the option of starting execution when the desired value of the probability
   of covered contingencies has been reached. 
   \end{proof}

\section{Empirical Evaluation}
\label{sec:empirical}

\subsection{Experimental Setup}
We use a total of five domains with varying configurations to evaluate our approach. All these five domains had  a mix of deterministic and stochastic actions. We use an implementation of LAO*~\cite{hansen2001lao} from the MDP-Lib~\cite{mdplib} repository for computing policies for SSPs. We use  OpenRAVE~\cite{diankov10_openrave} robot simulation system with its collision checkers to represent 3D environments and performing collision checking. We also CBiRRT's~\cite{bernson2009cbirrt} implementation from the PrPy~\cite{prpy} suite for computing motion plans. In practice, fixing the horizon $H$ for the SSP solver apriori is infeasible and renders some problems unsolvable. Instead, we implemented a variant that dynamically increases the horizon until the goal is reached with a probability $p > 0$. The source code of the framework along with the videos of our experiments can be found at \url{https://aair-lab.github.io/STAMP.html}

\citet{lagriffoul2018platform} propose several framework-independent benchmark domains for task and motion planning systems. While these benchmarks are proposed for deterministic TAMP systems, characteristics of the domains can still be used to evaluate STAMP systems. Fig. 
\ref{fig:char} shows the criteria fulfilled by every domain used to evaluate our approach. We include the average number of branches in the policy tree as an additional criterion to depict the complexity of stochastic problems. 

\begin{figure*}[t!]
\begin{subfigure}{1\columnwidth}
  \centering
  \includegraphics[width=0.3\columnwidth]{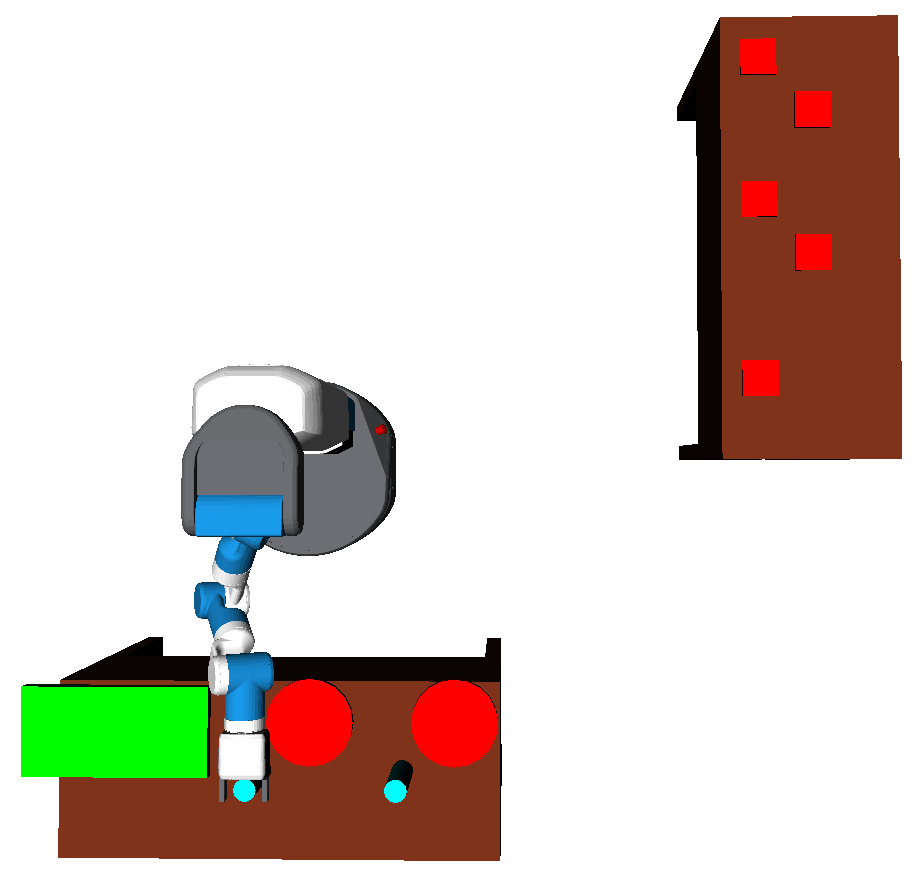}
  % \hspace{1em}
  \includegraphics[width=0.3\columnwidth]{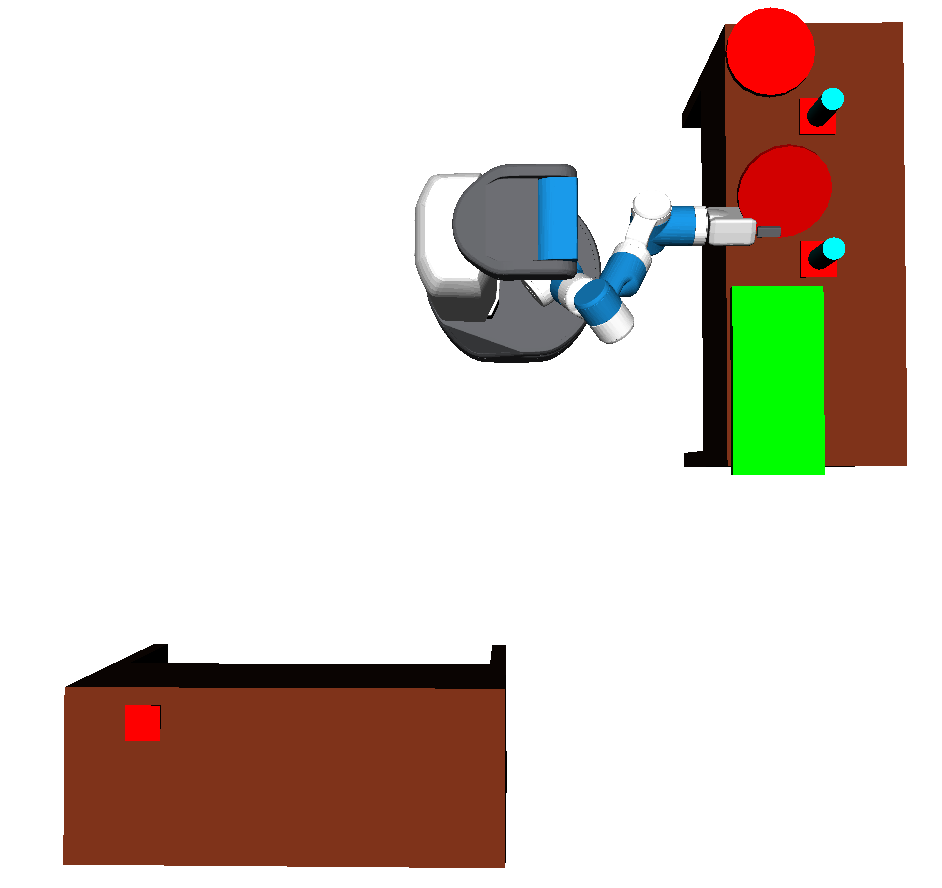}

\end{subfigure}
\caption{Setting up a dining table: Fetch uses STAMP policy to set up a dining table. A tray is available to carry multiple items at a time but carrying more than two items on the tray may break the items. Left: The initial state. Right: The goal state.}
\label{fig:kitchen}

\end{figure*}

\paragraph{\textbf{Problem $\mathbf{1}$: Cluttered Table}}
In this problem, we have a table cluttered with cans, each having different probabilities of being crushed when grasped by the robot. Some cans are delicate and are highly likely to be crushed when the robot grabs them, incurring a high cost (probability for crushing was set to $0.1$, $0.5$ \& $0.9$ in different experiments in Fig. \ref{fig:anytime_result}(a)), while others are normal cans that cannot be crushed. The goal for the robot is to pick up a specified can. We used different numbers of cans ($15$, $20$, $25$) and different random configurations of cans to extensively evaluate the proposed framework. We also used this scenario to evaluate our approach in the real-world (Fig. \ref{fig:exp_1_2}) using the Fetch robot~\cite{wise16_fetch}.

% \begin{figure*}[t!]
%     \centering
% %    \includegraphics[width=\textwidth]{./anytime_results_new.eps}
%   % \includegraphics[width=\textwidth]{./anytime_results_new.png}
%   \includegraphics[width=\textwidth,height=7in]{./anytime_results_new_2.eps}
%     \caption{Anytime performance of ATM-MDP, showing the time in seconds (x-axis) vs. probability mass refined (y-axis).}
%     \label{fig:anytime_result}
%     % \vspace{-1em}
% \end{figure*}

\paragraph{\textbf{Problem $\mathbf{2}$: Aircraft Inspection}}
In this problem, an unmanned aerial vehicle (UAV) is employed to inspect possibly faulty parts of an aircraft in an airplane hangar.  The goal for the agent is to locate the fault and notify the human supervisor about it. Fig. \ref{fig:exp_3_6} shows the simulated environment. The UAV's sensors are inaccurate and may fail to locate the fault with some non-zero probability (failure probability was set to 0.05, 0.1, \& 0.15 for experiments in Fig. \ref{fig:anytime_result}(b)) while inspecting the location; it may also drift to another location while flying from one location to another or while inspecting the parts. The UAV has a limited amount of battery charge. A charging station is available for the UAV to dock and charge itself. All movements use some amount of battery charge depending on the length of the trajectory, but the high-level planner cannot determine whether the current level of charge is sufficient for the action or not as it lacks the details such as current battery level, length of previous and next trajectories, etc.  This makes it necessary to have an interleaved approach that searches for a high-level policy that has valid low-level refinements.

\begin{figure*}[t!]
  \footnotesize
  \begin{center}
  \begin{tabular}{|c|c|c|c|c|c|}
      \hline
      Criteria & Cluttered Table & \makecell{Aircraft \\ Inspection} &  \makecell{Building Keva \\ Structures} &  Kitchen & Find the can \\ \hline
      % Deterministic & \checkmark & ~ & \checkmark &  \checkmark  &  ~ & ~   \\ \hline
      % Stochastic & \checkmark & \checkmark & \checkmark & ~ & \checkmark & \checkmark \\ \hline    
      Infeasible Tasks & \checkmark & \checkmark & ~  &  ~ & ~  \\ \hline
      Large task spaces & \checkmark & \checkmark & \checkmark & \checkmark  & ~  \\ \hline
      Motion/task trade-off & \checkmark & \checkmark  & \checkmark & \checkmark & ~ \\ \hline
      Non-monotonicity & \checkmark & ~ & ~ & \checkmark &  \checkmark \\ \hline
      $\#$branches & $O(2d)$ & $O(4^h)$ & $O(2^n)$ &  $2$ & $2$  \\ \hline
  \end{tabular}
  \end{center}
  \caption{Critera defined by \cite{lagriffoul2018platform} evaluated in each of the test domains.}
  \label{fig:char}
  \end{figure*}

\paragraph{\textbf{Problem $\mathbf{3}$: Building Structures with Keva Planks}}
In this problem, the YuMi robot~\cite{yumi} is used to build different structures using Keva planks. Keva planks are laser-cut wooden planks with uniform geometry. Fig.\,\ref{fig:exp_1_2} and Fig.\,\ref{fig:domainFig} show the target structures.  Planks are placed one at a time by a user after each pickup and placement by the YuMi. Each new plank may be placed at one of a few predefined locations, which adds uncertainty in the planks' initial location. For our experiments, two predefined locations were used to place the planks with a probability of $0.8$ for the first location and a probability of $0.2$ for the second location. In this problem, handwritten goal conditions are used to specify the desired target structure. The YuMi needs a task and motion policy for successively picking up and placing planks to build the structure. There are infinitely many configurations in which one plank can be placed on another, but the abstract model blurs out different regions on the plank. The generator that samples put-down poses for planks on the table uses the target structure to concretize each plank's target put-down pose. The number of branches in a solution tree grows exponentially with the number of planks in the structure and can quickly become huge. For example, a solution tree for a structure with just $10$ planks would have a total of $1024$ branches. Due to the large state space, state-of-the-art SSP solver used for other domains failed to compute a high-level policy for these problems. Our observation shows that most SSP solvers fail to compute a high-level solution for structures that have greater than $6$ planks. However, these structure-building problems exhibit repeating substructure every 1-2 layers that reuse minor variants of the same abstract policy. We used this observation and used a generalized SSP solver \cite{karia2022preliminary} that computes generalized policies for SSPs with such repeating patterns. Other approaches for generalized planning~\cite{srivastava2008AAAI,bonet2009,hu2011,srivastava11_aij} can also be used to automatically extract and utilize such patterns in other problems with repeating structures.

% \paragraph{Problem $4$: Sort cluttered table} ~ \vspace{0.3em}\\
% In this problem, as shown in Fig. \ref{fig:sort}, the \emph{Fetch} robot is used to sort objects placed on a table. The goal is to have all $N$ blue blocks on the left table and all $N$ green blocks on the right table while $2N$ red blocks act as obstacles. The cluttered configuration of objects on the table renders some actions infeasible that makes ordering of the actions critical. The interleaved framework proposed by \emph{HPlan} algorithm searches for an abstraction that is sufficient to compute a valid order of the cans to be moved to solve the problem. For our experiments, we use $N = 3$, $5$, and $8$ (total number of objects 12,
%  20, and 32 respectively). 

\paragraph{\textbf{Problem $\mathbf{4}$: Setting Up a Dining Table}}
In this problem, the Fetch robot arranges a dining table with two plates and two glasses (Fig. \ref{fig:kitchen}). A tray is available for the robot to use for carrying multiple items at once. If the robot tries to carry more than two objects on a tray at once, the objects can fall from the tray with a probability $0.2$ and that would break the objects. While using the tray can reduce the number of trips between tables, breaking the objects would render the problem unsolvable. As our approach considers all possible outcomes of stochastic actions, it successfully computes a policy that prevents any object from breaking compared to determinization-based approaches that only consider the most likely outcome for stochastic actions that may fail to solve such problems as most-likely scenarios might fail to capture dead ends in the domain.

\begin{figure*}[t!]
  \centering
\includegraphics[width=\textwidth,height=6in]{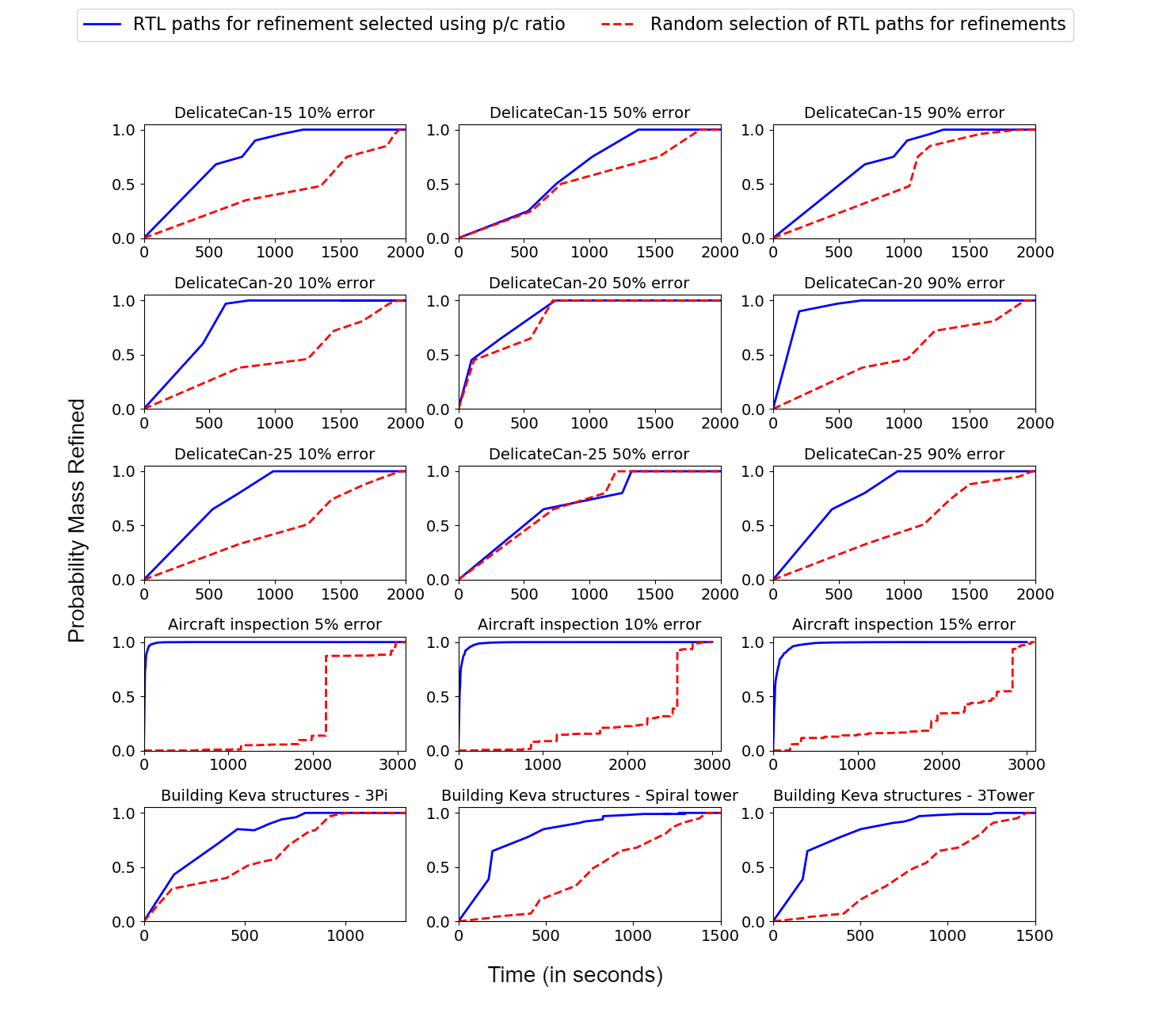}
 \caption{Anytime performance of ATM-MDP, showing the time in seconds (x-axis) vs. probability mass refined (y-axis).}
  \label{fig:anytime_result}
  % \vspace{-1em}
\end{figure*}

\paragraph{\textbf{Problem $\mathbf{5}$: Find the Can}} In this problem, the \emph{Fetch} robot searches for a can that may be present in one of the drawers. Fig. \ref{fig:exp_3_6} shows the simulated environment for the problem. The can is placed in one of the drawers with a given prior distribution. The robot does not have access to the can's location apriori and has to open the drawer to check whether the can is present in the drawer or not. In our experiments, the can is placed in the upper drawer with a probability $0.6$ and in the bottom drawer with a probability $0.4$. 

\subsection{Analysis of the results}

% \begin{figure}[t]
%   \centering 
%   \centering
% %   \footnotesize
%   \begin{tabular}{|l|c|c|}
%     \hline
%     Problem & \% Solved & Avg. Time (s)  \\ \hline 
%     Cluttered-15 &  100  & 367.89 $\pm$ 854.52  \\
%     Cluttered-20 &  97   & 654.1541 $\pm$ 1641.98   \\
%     Cluttered-25 & 86 & 990.93 $\pm$ 1011.07  \\ 
%     Aircraft Inspection & 100 & 278.94 $\pm$ 30.54 \\ 
%     $3\pi$ &  100 & 227.04 $\pm$ 38.11   \\
%     Twisted-Tower-12 &  100 &  805.31 $\pm$ 102.10  \\
%     Three-Tower-12 & 100  & 1367.27 $\pm$ 144.29  \\
%     Sort Clutter ($N=3$) & 100 & 687.65 $\pm$ 103.36 \\
%     Sort Clutter ($N=5$) & 52 & 2384.91 $\pm$ 540.65 \\
%     Sort Clutter ($N=8$) & 12 & 3687.65 $\pm$ 301.21 \\
%     \hline
%   \end{tabular}
%   \caption{Summary of times taken to solve the TAMP problems. Timeout: 4000 seconds}
%   \label{fig:time_result_deterministic}
% \end{figure}

\begin{figure}[t]
    \centering
  \footnotesize
  \begin{tabular}{|l|c|c|}
          \hline
          Problem & \% Solved & Avg. Time (s)  \\ \hline
          
          Cluttered-15 & 100 & 1120.21 $\pm$ 1014.54  \\
          Cluttered-20 &  83 & 1244.32 $\pm$ 990.65  \\
          Cluttered-25 & 75 & 1684.54  $\pm$ 890.78  \\
          Aircraft Inspection & 100 & 2875.01 $\pm$ 103.65 \\
          $3\pi$ &  100 & 1356.34 $\pm$ 75.8    \\
          Tower-12 &  100 &  2232.36 $\pm$ 104.84  \\
          Twisted-Tower-12 & 80 & 3249.92 $\pm$ 773.69   \\ 
          Setting up a dining table & 100 & 1287.23 $\pm$ 321.32 \\
          Find the can & 100 & 36.74 $\pm$ 0.13\\ 

      \hline
      \end{tabular}
      \caption{Summary of times taken to solve the STAMP problems. Timeout: 4000 seconds.}
      \label{fig:time_result_stochastic}
      % \vspace{-2em}
  \end{figure}

% policies that consider all possible contingencies
% nature of solutions 
\paragraph{\textbf{Nature of the Solutions }}
The most distinct characteristic of the solutions generated through our framework is that they capture all possible contingencies that may arise while executing the policy. E.g., solutions generated for setting up the dinner table (problem 4) avoid placing more than two items on the tray to completely eliminate the possibility of incurring higher expected cost, and solutions for picking up a can from the cluttered table (problem 1) avoid picking up a delicate can for similar reasons. 

\paragraph{\textbf{Quality of the Solutions Over Time}} 
While our approach computes refinements for every action in the policy, the anytime property allows the agent to start executing the actions before all the actions are refined. Our approach computes anytime policies with respect to the possible outcomes handled by a policy at any point in time. Fig. \ref{fig:anytime_result} shows the anytime property of our approach in stochastic test domains. The y-axis shows the probability with which the policy available at any point of time during the algorithm's computation will be able to handle all possible outcomes, and the x-axis shows the time (in seconds) required to compute task and motion policies that handle these outcomes. The results show that with time, the likelihood with which the solution would be able to handle any scenario increases. The agent can use this observation to decide a threshold at which it can start executing the actions. For our experiments, we use a threshold of $60\%$ of all possible outcomes to start the execution of the policy. Our experiments show that in most cases, the problem was solved significantly faster compared to starting execution after refining the entire policy tree (Fig.~\ref{fig:time_result_stochastic}).

\paragraph{\textbf{Impact of Prioritized \emph{RTL} Path Selection}} 
The results presented in Fig. \ref{fig:anytime_result} indicate that when \emph{RTL} paths are selected using the $p/c$ ration (blue line), the framework can quickly handle outcomes with most likely outcomes, compared to a randomized selection of \emph{RTL} paths for refinements (red line). In most cases, $80\%$ of probable executions are covered within about $30\%$ of the total computation time. This characteristic is most evident in the \emph{aircraft inspection} problem due to a large number of possible outcomes and differences in the probability of different outcomes. Such a prioritization does not make a significant impact if all the outcomes are equally probable. E.g., such impact is the least evident in the \emph{cluttered table} problem with the probability of crushing the objects set to $0.5$ given each outcome becomes equally probable and the sequence in which they are handled does not make any difference. 

\paragraph{\textbf{Scalability of the Framework}}
Fig.~\ref{fig:time_result_stochastic} shows the time taken by our approach to compute  complete STAMP solutions by concretizing every action in the entire policy for the given test problems respectively. We combine results for different variants of the test problem as variations in the probabilities of outcomes do not affect the time required to concretize all actions in the entire policy. Values in Fig.~\ref{fig:time_result_stochastic} are averages of $50$ runs with standard deviation. Our empirical evaluation shows that solving a STAMP problem requires significantly more time than  an equivalent TAMP problem. E.g., the stochastic variant of the aircraft inspection problem takes nearly $15$ times more time than the deterministic version as the stochastic variant had $780$ branches in the solution tree compared to a single branch in the deterministic variant. These results reinforce our hypothesis that an anytime approach that prioritizes high-probability scenarios over low-probability situations but still considers all possible outcomes suits better than an approach that does not consider all possible outcomes while showing scalability of our approach to solve large problems. Results for larger problems such as \emph{Twisted-Tower-12} and \emph{Cluttered-25} show \emph{scalability} of our system. Even though our approach needs a significant time to compute solutions for such huge problems due to a large number of RTL paths in the policy trees, it was able to solve almost all problems in these problem settings.

% ~ \newpage~ \newpage

% \input{problem.tex}

% \input{conclusion.tex}

\section*{Acknowledgements}
This work was supported in parts by the NSF under grants IIS 1844325, IIS 1909370, and OIA 1936997.

% ~ \newpage
% ~ \newpage
%  \bibliographystyle{SageH}
\bibliography{ref}
\bibliographystyle{theapa}

z

\end{document}